\documentclass[10pt]{article}
\usepackage[letterpaper,margin=1in]{geometry}

\usepackage[T1]{fontenc}    %
\usepackage[english]{babel}
\usepackage{csquotes}
\usepackage{url}            %
\usepackage{booktabs}       %
\usepackage{amsmath,amsbsy,amsfonts,amssymb,amsthm,color,dsfont,mleftright}

\usepackage[theoremfont]{newpxtext}
\usepackage{cabin}
\usepackage[bigdelims,vvarbb]{newpxmath}

\usepackage{bm}
\usepackage{mathrsfs}

\allowdisplaybreaks
\usepackage[toc,page]{appendix}

\usepackage{thmtools}
\usepackage{mathtools}
\usepackage{nicefrac}       %
\usepackage{graphicx}
\usepackage[
backend=biber,
style=alphabetic,
sorting=ynt,
sortcites,
doi=false,
hyperref=true,
isbn=false,
url=false,
maxbibnames=99,
maxalphanames=4,
minalphanames=4
]{biblatex}
\renewbibmacro{in:}{}
\AtEveryBibitem{%
  \clearlist{language}%
}
\usepackage{hyperref}       %
\usepackage{nameref}
\usepackage[capitalise,noabbrev,sort]{cleveref}
\usepackage{enumitem}
\usepackage[dvipsnames]{xcolor}
\usepackage{wrapfig}
\usepackage[normalem]{ulem}
\usepackage[bf,small]{caption}
\usepackage{subcaption}
\usepackage{overpic}

\usepackage{algorithm}
\usepackage[noend]{algorithmic}

\newcommand{\mycolor}{RoyalBlue}
\newcommand{\mycitecolor}{Red}

\hypersetup{
  linkcolor = \mycolor,
  citecolor = \mycitecolor,
  urlcolor = \mycitecolor,
  colorlinks = true,
  linktocpage = true,
}

\usepackage{tikz}
\usepackage{pgfplots}
\usepackage[most]{tcolorbox}

\crefname{equation}{}{}

\newtheorem{theorem}{Theorem}[section]

\newtheorem{proposition}[theorem]{Proposition}

\theoremstyle{remark}

\theoremstyle{definition}

\input{sam_macros.def}

\newcommand{\mb}{\mathbf}
\renewcommand{\mathbf}{\boldsymbol}
\newcommand{\mc}{\mathcal}

\newcommand{\bb}{\mathbb}
\newcommand{\innerprod}[2]{\left\langle #1,  #2 \right\rangle}

\newcommand{\codeurl}{\url{https://github.com/sdbuch/refine}}

\title{Resource-Efficient Invariant Networks: Exponential Gains by Unrolled Optimization}
\author{%
  Sam Buchanan\thanks{Corresponding author:
    \texttt{s.buchanan@columbia.edu}}\,\,\,\thanks{Department of Electrical Engineering, Columbia
  University}\,\,\,\thanks{Data Science Institute, Columbia University} \qquad %
  Jingkai Yan\footnotemark[2]\,\,\,\footnotemark[3] \qquad %
  Ellie Haber\thanks{NYU Tandon School of Engineering} \qquad %
  John Wright\footnotemark[2]\,\,\,\footnotemark[3]\,\,\,\thanks{Department of Applied Physics and Applied Mathematics, Columbia University}
}

\begin{document}

\maketitle

\begin{abstract}
  Achieving invariance to nuisance transformations
is a fundamental challenge in the construction of robust and reliable vision systems.
Existing approaches to invariance scale exponentially with the dimension of the
family of transformations, making them unable to cope with natural variabilities in visual data
such as changes in pose and perspective.
We identify a common limitation of these approaches---they rely on sampling to traverse the
high-dimensional space of transformations---and propose a new computational primitive for building
invariant networks based instead on optimization, which in many scenarios provides a provably more
efficient method for high-dimensional exploration than sampling.
We provide empirical and theoretical corroboration of the efficiency gains and soundness of our
proposed method, and demonstrate its utility in constructing an efficient invariant network for a
simple hierarchical object detection task when combined with unrolled optimization.  Code for our
networks and experiments is available at \codeurl.

\end{abstract}

\section{Introduction}\label{sec:intro}

\begin{figure*}[!htb]
  \centering
  \includegraphics[width=\linewidth]{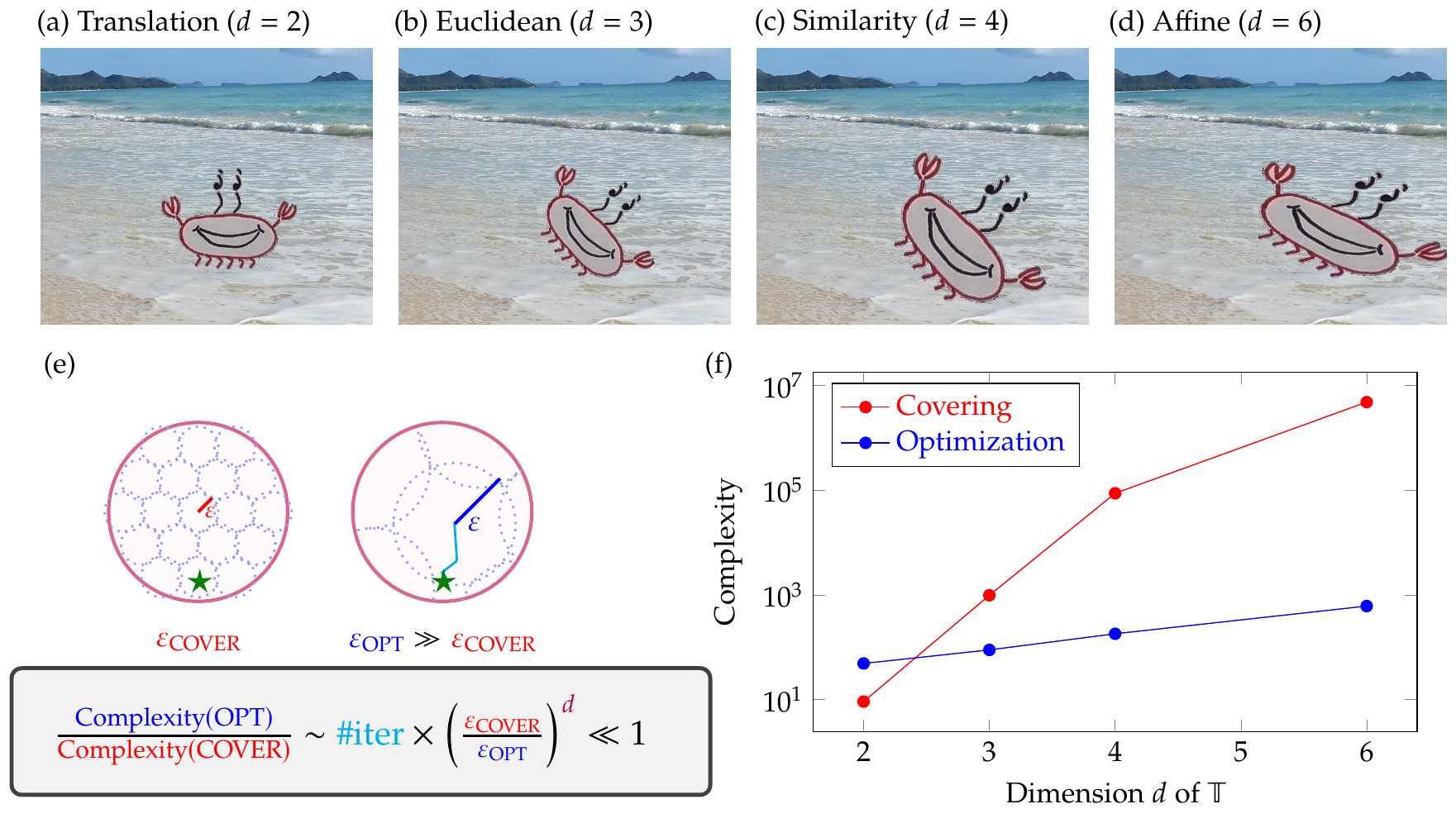}
  \caption{Comparing the complexity of covering-based and optimization-based methods for invariant
    recognition of a template embedded in visual clutter.  \textbf{(a-d)}: We consider four
    different classes of deformations that generate the observation of the template, ranging
    across shifts, rotations, scale, and skew. The dimension $d$ of the family of transformations
    increases from left to right. \textbf{(e)}: A geometric illustration of the covering and
    optimization approaches to global invariance: in certifying that a query (labeled with a star)
    is a transformed instance of the template (at the base point of the solid red/blue lines),
    optimization can be vastly more efficient than covering, because it effectively covers the
    space at the \textit{scale of the basin of attraction} of the optimization problem, which is
    always larger than the template's associated $\veps_{\mathrm{COVER}}$. %
    \textbf{(f)}: Plotting the average number of convolution-like operations necessary to reach a
    zero-normalized cross-correlation (ZNCC) of 0.9 between the template and a
    randomly-transformed query across the different deformation classes. Optimization leads to an
    efficiency gain of several orders of magnitude as the dimensionality of the family of
    transformations grows. Precise experimental details are recorded in
  \Cref{sec:extra_experiment_details}.}
\label{fig:centerpiece}
\end{figure*}

In computing with any kind of realistic visual data, one must contend with a dizzying array of
complex variabilities: statistical variations due to appearance and shape, geometric variations
due to pose and perspective, photometric variations due to illumination and cast shadows, and
more. Practical systems cope with these variations by a data-driven approach, with deep  neural network architectures
trained on massive datasets. This approach is especially successful at coping with variations in texture
and appearance. For invariance to geometric transformations of the input
(e.g., translations, rotations, and scaling, as in \Cref{fig:centerpiece}(a-d)), the predominant
approach in practice is also data-driven: %
the `standard pipeline' is to deploy an architecture that
is structurally invariant to translations (say, by virtue of convolution and pooling), and improve
its stability with respect to other types of transformations by data augmentation. Data
augmentation generates large numbers of synthetic training samples by applying various
transformations to the available training data, and demonstrably contributes to the performance of
state-of-the-art systems \cite{Chen2020-bl,Cubuk2018-yu,Hendrycks2019-ae}. However, it runs into a
basic resource efficiency barrier associated with the dimensionality of the set of nuisances: {\em
learning over a $d$-dimensional group of transformations requires both data and architectural
resources that are {\em exponential} in $d$}
\cite{Chen2019-eq,Basri2017-vn,Nakada2019-jr,Schmidt-Hieber2019-jt,Cloninger2020-ov}.  This is a
major obstacle to achieving invariance to large, structured deformations such as 3D rigid body
motion ($d = 6$), homography ($d = 8$), and linked rigid body motion \cite{Kanazawa_2019_CVPR} and
even nonrigid deformations \cite{Zaitsev2015-ot} ($d \gg 8$). It is no surprise,
then, that systems trained in this fashion remain %
vulnerable to adversarial
transformations of domain \cite{Azulay2018-nm, Alcorn2018-qg, Alaifari2018-ef, Fawzi:210209,
Engstrom2017-ps, Kanbak2018-lw, Xiao2018-eq}---it simply is not possible to generate enough
artificial training data to learn transformation manifolds of even moderate dimension.  Moreover,
this approach is fundamentally wasteful: learning nuisances known to be present in the input data
wastes architectural capacity that would be better spent coping with statistical variability in
the input, or learning to perform complex tasks.

These limitations of the standard pipeline are well-established, and they have inspired a range of
alternative \textit{architectural} approaches to achieving invariance, where each layer of the
network incorporates computational operations that reflect the variabilities present in the data.
Nevertheless, as we will survey in detail in \Cref{sec:relwork}, all known approaches are subject
to some form of exponential complexity barrier: the computational primitives demand either a
filter count that grows as $\exp(d)$ or integration over a $d$-dimensional space, again incurring
complexity exponential in $d$. Like data augmentation, these approaches can be seen as obtaining
invariance by exhaustively sampling transformations from the $d$-dimensional space of nuisances,
which seems fundamentally inefficient: in many concrete high-dimensional signal recovery problems,
\textit{optimization} provides a significant advantage over naive grid searching when exploring a
high-dimensional space \cite{Gilboa2019-px,Simchowitz2017-pe,Carmon2021-rl}, as in
\Cref{fig:centerpiece}(e). This motivates us to ask:
\begin{quote}
  \begin{center}
    \textbf{Can we break the barrier between resource-efficiency and invariance using
    \textit{optimization} as the architectural primitive, rather than sampling?}
  \end{center}
\end{quote}
In \Cref{fig:centerpiece}, we conduct a simple experiment that suggests a promising avenue to answer
this question in the affirmative. Given a known
synthetic textured motif subject to an unknown structured transformation and embedded in a
background, we calculate the number of computations (convolutions and interpolations) required to
certify with high confidence that the motif appears in the image. Our baseline approach is
template matching, which enumerates as many transformations of the input as are necessary to
certify the motif's occurrence (analogous to existing architectural approaches with
sampling/integration as the computational primitive)---each enumeration requires one interpolation
and one convolution.  We compare to a gradient-based optimization approach that attempts to match
the appearance of the test image to the motif, which uses three interpolations and several
convolutions per iteration (and on the order of $10^2$ iterations).
As the dimensionality of the space of transformations grows, the optimization-based approach
demonstrates an increasingly-significant efficiency advantage over brute-force enumeration of
templates---at affine transformations, for which $d=6$, it becomes challenging to even obtain a suitable transformation of the template by sampling.

To build from the promising optimization-based approach to local invariance in this experiment to
a full invariant neural network architecture capable of computing with realistic visual data, 
one needs a general method to incorporate prior information about the specific visual data,
observable transformations, and target task into the design of the network. We take the first
steps towards realizing this goal: inspired by classical methods for image
registration in the computer vision literature, we propose an optimization formulation for seeking
a structured transformation of an input image that matches previously-observed images, and we show
how combining this formulation with \textit{unrolled optimization}
\cite{Gregor2010-nh,Chen2021-un,Ongie2020-pu}, which converts an iterative solver for an
optimization problem into a neural network, implies resource-efficient and principled invariant
neural architectural primitives. In addition to providing network
architectures incorporating `invariance by design', this is a principled approach that leads to
networks amenable to theoretical analysis, and in particular we provide convergence guarantees for
specific instances of our optimization formulations that transfer to the corresponding unrolled
networks. On the practical side, we illustrate how these architectural primitives can be combined
into a task-specific neural network by designing and evaluating an invariant network architecture
for an idealized single-template hierarchical object detection task, and present an experimental
corroboration of the soundness of the formulation for invariant visual motif recognition used in
the experiment in \Cref{fig:centerpiece}. Taken altogether, these results demonstrate a promising
new direction to obtain theoretically-principled, resource-efficient neural networks that achieve
guaranteed invariance to structured deformations of image data.

The remainder of the paper is organized as follows: \Cref{sec:relwork} surveys the broad range of
architectural approaches to invariance that have appeared in the literature; \Cref{sec:framework}
describes our proposed optimization formulations and the unrolling approach to network design;
\Cref{sec:motif} describes the hierarchical invariant object detection task and a corresponding
invariant network architecture; \Cref{sec:theory} establishes convergence guarantees for our
optimization approach under a data model inspired by the hierarchical invariant object detection
task; and \Cref{sec:basin_experiments} provides a more detailed look at the invariance
capabilities of the formulation used in \Cref{fig:centerpiece}.

\section{Related Work}\label{sec:relwork}

\paragraph{Augmentation-based invariance approaches in deep learning.}

The `standard pipeline' to achieving invariance in deep learning described in \Cref{sec:intro}
occupies, in a certain sense, a minimal point on the tradeoff curve between a purely data-driven
approach and incorporating prior knowledge about the data into the architecture: by using a
convolutional neural network with pooling, invariance to translations of the input image (a
two-dimensional group of nuisances) is (in principle) conferred, and invariance to more complex
transformations is left up to a combination of learning from large datasets and data augmentation.
A number of architectural proposals in the literature build from a similar perspective, but occupy
different points on this tradeoff curve. Parallel channel networks \cite{Ciresan2012-ha} generate
all possible transformations of the input and process them in parallel, and have been applied for
invariance to rotations \cite{10.1093/mnras/stv632, Laptev_2016_CVPR} and scale
\cite{DBLP:journals/corr/abs-2106-06418}. Other architectures confer invariance by pooling over
transformations at the feature level \cite{Sohn2012-ml}, similarly for rotation
\cite{Worrall_2017_CVPR} and scale \cite{Kanazawa2014-zf}. Evidently these approaches become
impracticable for moderate-dimensional families of transformations, as they suffer from the same
sampling-based bottleneck as the standard pipeline.

To avoid explicitly covering the space of transformations, one can instead incorporate learned
deformation offsets into the network, as in deformable CNNs \cite{Dai2017-lj} and spatial
transformer networks \cite{10.5555/2969442.2969465}.
At a further level of generality, capsule
networks \cite{Hinton2011-yb,SabourFH17, Sun2020-mu, Hinton2021-fm} allow more flexible
deformations among distinct parts of an object to be modeled. 
The improved empirical performance observed
with these architectures in certain tasks illustrates the value of explicitly modeling
deformations in the network architecture. At the same time, when it comes to guaranteed
invariance to specific families of structured deformations, they suffer from the same exponential
inefficiencies as the aforementioned approaches. 

\paragraph{Invariance-by-construction architectures in deep learning.}

The fundamental efficiency bottleneck encountered by the preceding approaches has motivated the
development of alternate networks that are invariant simply by virtue of their constituent
computational building blocks. Scattering networks \cite{Bruna2013-on} are an especially
principled and elegant approach: they repeatedly iterate layers that convolve an input signal with
wavelet filters, take the modulus, and pool spatially. These networks provably obtain
translation invariance in the limit of large depth, with feature representations that are
Lipschitz-stable to general deformations \cite{Mallat2012-sl}; moreover, the construction and
provable invariance/stability guarantees generalize to feature extractors beyond wavelet
scattering networks \cite{Wiatowski2018-qr}. Nevertheless, these networks suffer from a similar
exponential resource inefficiency to those that plague the augmentation-based approaches: each
layer takes a wavelet transform of \textit{every} feature map at the previous layer, resulting in
a network of width growing exponentially with depth. Numerous mitigation strategies have been
proposed for this limitation \cite{Bruna2013-on, Zarka2020Deep, zarka2021separation}, and
combinations of relatively-shallow scattering networks with standard learning machines have
demonstrated competitive empirical performance on certain benchmark datasets
\cite{Oyallon_2017_ICCV}. However, the resulting hybrid networks still suffer from an inability to
handle large, structured transformations of domain such as pose and perspective changes.

Group scattering networks attempt to remedy this deficiency by replacing the spatial convolution
operation with a group convolution $w, x \mapsto [w \ast x](g) = \int_{\bbG} x(g') w(g\inv g')
\diff \mu(g')$ \cite{Mallat2012-sl,Cohen2016-vq,pmlr-v80-kondor18a,Bronstein2021-un}. In this
formula, $\bbG$ is a group with sufficient topological structure, $\mu$ is Haar measure on $\bbG$,
and $w$ and $x$ are the filter and signal (resp.), defined on $\bbG$ (or a homogeneous space for
$\bbG$, as in spherical CNNs \cite{Cohen2018-aj}). Spatial convolution of natural images coincides
with the special case $\bbG = \bbZ^2$ in this construction; for more general groups such as 3D
rotation, networks constructed by iterated group convolutions yield feature representations
equivariant to the group action, and intermixing pooling operations yields invariance, just as
with 2D convolutional neural networks.  At a conceptual level, this basic construction implies
invariant network architectures for an extremely broad class of groups and spaces admitting group
actions \cite{Weiler2021-ps}, and has been especially successful in graph-structured tasks such as
molecular prediction where there is an advantage to enforcing symmetries \cite{Bronstein2021-un}.
However, its application to visual data has been hindered by exponential inefficiencies in
computing the group convolution---integration over a $d$-dimensional group $\bbG$ costs resources
exponential in $d$---and more fundamentally by the fact that discrete images are defined on the
image plane $\bbZ^2$, whereas group convolutions require the signal to be defined over the group
$\bbG$ one seeks invariance to. In this sense, the `reflexivity' of spatial convolution and
discrete images seems to be the exception rather than the rule, and there remains a need for
resource-efficient architectural primitives for invariance with visual data.

\paragraph{``Unrolling'' iterative optimization algorithms.}

First introduced by Gregor and LeCun in the context of sparse coding \cite{Gregor2010-nh},
unrolled optimization provides a general method to convert an iterative algorithm for solving
an optimization problem into a neural network (we will provide a concrete demonstration in
the present context in \Cref{sec:framework}), offering the possibility to combine the statistical
learning capabilities of modern neural networks with very specific prior information about the
problem at hand \cite{Chen2021-un}. It has found broad use in scientific imaging and engineering
applications \cite{Kobler2017-wr,Buchanan2018-ip,8667888, 8747339,Ongie2020-pu}, and most
state-of-the-art methods for learned MRI reconstruction are based on this approach \cite{9420272}.
In many cases, the resulting networks are amenable to theoretical analysis
\cite{liu2018alista,10.5555/3327546.3327581}, leading to a mathematically-principled neural
network construction.

\section{Invariant Architecture Primitives: Optimization and Unrolling}
\label{sec:framework}

\paragraph{Notation.} We write $\bbR$ for the reals, $\bbZ$ for the integers, and $\bbN$ for the positive integers.
For positive integers $m$, $n$, and $c$, we let $\bbR^{m}$, $\bbR^{m \times n}$, and $\bbR^{m
\times n \times c}$ denote the spaces of real-valued $m$-dimensional vectors, $m$-by-$n$ matrices,
and $c$-channel $m$-by-$n$ images (resp.). We write $\ve_{i}$, $\ve_{ij}$, etc.\ to denote the
elements of the canonical basis of these spaces, and $\One_m$ and $\Zero_{m,n}$ (etc.) to denote
their all-ones and all-zeros elements (resp.). 
We write $\ip{}{}$ and $\norm{}_F$ to denote the euclidean inner product
and associated norm of these spaces.
We identify $m$ by $n$ images $\vx$ with functions on the integer grid $\set{0, 1, \dots, m-1}
\times \set{0, 1, \dots, n-1}$, and therefore index images starting from $0$; when applying
operations such as filtering, we will assume that an implementation takes the necessary zero
padding, shifting, and truncation steps to avoid boundary effects.
For a subset $\Omega \subset \bbZ^2$, we write $\sP_{\Omega}$ for the orthogonal projection onto
the space of images with support $\Omega$.

Given a deformation vector field $\vtau \in \bbR^{m' \times n' \times 2}$ and an image $\vx \in
\bbR^{m \times n \times c}$, we define the transformed image $\vx \circ \vtau$ by $(\vx \circ
\vtau)_{ij} = \sum_{(k, l) \in \bbZ^2} \vx_{kl} \phi(\tau_{ij0} - k) \phi(\tau_{ij1} - l)$, where
$\phi : \bbR \to \bbR$ is the cubic convolution interpolation kernel
\cite{Keys1981-jx}.\footnote{The function $\phi$ is compactly supported on the interval $[-2, 2]$,
and differentiable with absolutely continuous derivative.} %
For parametric transformations of the image plane, we write $\vtau_{\vA,
\vb}$ to denote the vector field representation of the transformation parameterized by $(\vA,
\vb)$, where $\vA \in \bbR^{2\times 2}$ is nonsingular and $\vb \in \bbR^2$ (see
\Cref{sec:parametric} for specific `implementation' details).
For two grayscale images $\vx \in \bbR^{m \times n}$ and $\vu \in \bbR^{ m' \times n'}$, we write
their linear convolution as $(\vx \conv \vu)_{ij} = \sum_{(k, l) \in \bbZ^2} x_{kl} u_{i-k, j-l}$.
We write $\vg_{\sigma^2} \in \bbR^{\bbZ \times \bbZ}$ to denote a (sampled) gaussian with zero
mean and variance $\sigma^2$.
When $\vx \in \bbR^{m \times n}$ and $\vu \in \bbR^c$, we write $\vx \kron \vu \in \bbR^{m \times
n \times c}$ to denote the `tensor product' of these elements, with $(\vx \kron \vu)_{ijk} =
x_{ij} u_k$. We use $\vx \circleddot \vu$ to denote elementwise multiplication of images.

\subsection{Conceptual Framework}

Given an input image $\vy \in \bbR^{m \times n \times c}$ (e.g., $c=3$ for RGB images), we consider
the following general optimization formulation for seeking a structured transformation of the
input that explains it in terms of prior observations:
\begin{equation}
  \min_{\vtau}\, \varphi(\vy \circ \vtau) + \lambda \rho(\vtau).
  \label{eq:general_formulation}
\end{equation}
Here, $\vtau \in \bbR^{m' \times n' \times 2}$ gives a vector field representation of
transformations of the image plane, and $\lambda > 0$ is a regularization tradeoff parameter. 
Minimization of the function $\varphi$ encourages the transformed input image $\vy \circ \vtau$ to
be similar to previously-observed images, whereas minimization of $\rho$ regularizes the
complexity of the learned transformation $\vtau$. Both terms allow to incorporate significant
prior information about the visual data and task at hand, and an optimal solution $\vtau$ to
\Cref{eq:general_formulation} furnishes an invariant representation of the input $\vy$.

\subsection{Computational Primitive: Optimization for Domain Transformations}
\label{sec:formulations}

We illustrate the flexibility of the general formulation \Cref{eq:general_formulation} by
instantiating it for a variety of classes of visual data. In the most basic setting, we may
consider the registration of the input image $\vy$ to a known motif $\vx_o$ assumed to be present
in the image, and constrain the transformation $\vtau$ to lie in a parametric family of
transformations $\bbT$, which yields the optimization formulation
\begin{equation}
  \min_{\vtau}\,
  \frac{1}{2} \norm*{
    \sP_{\Omega}\left[
      \vg_{\sigma^2} \conv ( \vy \circ \vtau - \vx_o )
    \right]
  }_F^2
  + \chi_{\bbT}(\vtau).
  \label{eq:formulation_reg}
\end{equation}
Here, $\Omega$ denotes a subset of the image plane corresponding to the pixels on which the motif
$\vx_o$ is supported, $\vg_{\sigma^2}$ is a gaussian filter with variance $\sigma^2$ applied
individually to each channel, and $\chi_{\bbT}(\mb\tau)$ denotes the characteristic function for
the set $\bbT$ (zero if $\vtau \in \bbT$, $+\infty$ otherwise). The parameters in
\Cref{eq:formulation_reg} are illustrated in \Cref{fig:motif}(a-d).  We do not directly implement
the basic formulation \Cref{eq:formulation_reg} in our experiments, but as a simple model for the
more elaborate instantiations of \Cref{eq:general_formulation} that follow later it furnishes
several useful intuitions. For instance, although \Cref{eq:formulation_reg} is a nonconvex
optimization problem with a `rough' global landscape, well-known results 
suggest that under idealized conditions (e.g., when $\vy = \vx_o \circ \vtau_o$ for some $\vtau_o
\in \bbT$), multiscale solvers that repeatedly solve \Cref{eq:formulation_reg} with a smoothing
level $\sigma^2_k$ then re-solve initialized at the previous optimal solution with a finer
level of smoothing $\sigma_{k+1}^2 < \sigma_k^2$ converge in a neighborhood of the
true transformation \cite{Lefebure2001-uz,Mobahi2012-to,KARYGIANNI2014232,Vural2014-xr}. This
basic fact underpins many classical computer vision methods for image registration and stitching
\cite{Brown1992-vz,Maintz1998-vz,Szeliski2007-ul, Baker2004-fh}, active appearance models for
objects \cite{Cootes1998-sk}, and optical flow estimation
\cite{Horn1981-by,Anandan1989-su,Lucas1981-ef}, and suggests that \Cref{eq:formulation_reg} is a
suitable base for constructing invariant networks.

\begin{figure}[!htb]
  \centering
    \includegraphics[width=\linewidth]{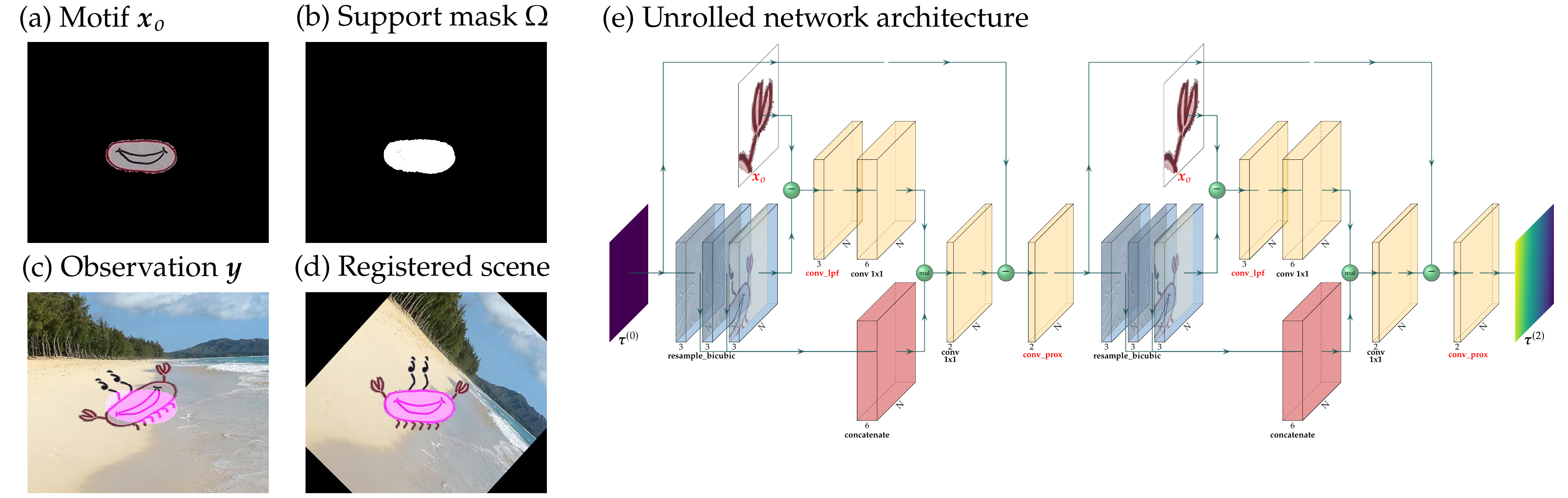}
    \caption{Motif registration with the formulation \Cref{eq:general_formulation}, and an
      unrolled solver. \textbf{(a-d)}: Visualization of components of a registration problem, such
      as \Cref{eq:formulation_reg}. We model observations $\vy$ as comprising an object involving
      the motif of interest (here, the body of the crab template we experiment with in
      \Cref{sec:motif}) on a black background, as in (a), embedded in visual clutter (here, the
      beach background) and subject to a deformation, which leads to a novel pose. A mask $\Omega$
      for the nonzero pixels of the motif, as in (b), is used to avoid having pixels corresponding
      to clutter enter the registration cost. After solving this optimization problem, we obtain a
      transformation $\vtau$ that registers the observation to the motif, as in (d). In (c-d), we
      set the red and blue pixels corresponding to the mask $\Omega$ to $1$ in order to visualize
      the relative location of the motif. \textbf{(e)}: Optimization formulations imply network
      architectures, via unrolled optimization. Here we show two iterations of an unrolled solver
      for \Cref{eq:formulation_reg}, as we detail in \Cref{sec:unrolling}; parameters that could
      be learned from data, \`{a} la unrolled optimization, are highlighted with red text. The
      operations comprising this unrolled network consist of linear maps, convolutions, and
    interpolations, leading to efficient implementation on standard hardware accelerators.}
  \label{fig:motif}
\end{figure}

For our experiments on textured visual data in \Cref{fig:centerpiece} and
\Cref{sec:motif}, we will need two elaborations of \Cref{eq:formulation_reg}. The first arises due
to the problem of obtaining invariant representations for images containing motifs $\vx_o$
appearing in general backgrounds: in such a scenario, the input image $\vy$ may contain the motif
$\vx_o$ in a completely novel scene (as in \Cref{fig:motif}(c-d)), which makes it
inappropriate to smooth the entire motif with the filter $\vg_{\sigma^2}$. In these scenarios, we
consider instead a \textit{cost-smoothed} formulation of \Cref{eq:formulation_reg}:
\begin{equation}
  \min_{\vtau}\,
  \frac{1}{2}
  \sum_{\vDelta \in \bbZ \times \bbZ}
  (\vg_{\sigma^2})_{\vDelta}
  \norm*{
    \sP_{\Omega}\left[
      \vy \circ (\vtau + \vtau_{\Zero, \vDelta}) - \vx_o
    \right]
  }_F^2
  + \chi_{\bbT}(\vtau).
  \label{eq:formulation_costsmooth}
\end{equation}
In practice, we take the sum over a finite subset of shifts $\vDelta$ on which most of the mass of
the gaussian filter lies.
This formulation is inspired by more general cost-smoothing registration proposals in the
literature \cite{Mobahi2012-to}, and it guarantees that pixels of $\vy \circ \vtau$ corresponding
to the background $\Omega^c$ are never compared to pixels of $\vx_o$ while incorporating the
basin-expanding benefits of smoothing. Second, we consider a more general formulation which also
incorporates a low-frequency model for the image background:
\begin{equation}
  \min_{\mb \tau,\, \mb \beta } \frac{1}{2} \left\| \mc
  P_{\wt{\Omega}}\left[ 
    \vg_{\sigma^2} \conv ( \vy \circ \mb \tau - \mb x_o - \mc P_{\Omega^c}[ \vg_{C\sigma^2} \conv
    \mb \beta ])
  \right]
  \right\|_F^2 + \chi_{\bbT}(\vtau).  
  \label{eq:formulation_withbg}
\end{equation}
Here, $\vbeta \in \bbR^{m \times n \times c}$ acts as a learnable model for the image background,
and $C > 1$ is a fixed constant that guarantees that the background model is at a coarser scale
than the motif and image content. The set $\wt{\Omega}$ represents a dilation by $\sigma$ of the
motif support $\Omega$, and penalizing pixels in this dilated support ensures that an optimal
$\vtau$ accounts for both foreground and background agreement.  We find background modeling
essential in computing with scale-changing transformations, such as affine transforms in
\Cref{fig:centerpiece}.

\subsection{Invariant Networks from Unrolled Optimization} \label{sec:unrolling}
The technique of unrolled optimization allows us to obtain principled network architectures from
the optimization formulations developed in \Cref{sec:formulations}. We describe the basic approach
using the abstract formulation \Cref{eq:general_formulation}. For a broad class of regularizers
$\rho$, the proximal gradient
method \cite{Parikh2014-om} can be used to attempt to solve the nonconvex problem
\Cref{eq:general_formulation}: it defines a sequence of iterates
\begin{equation}
  \mb \tau^{(t+1)} = \mr{prox}_{\lambda \nu_t \rho}\left( \mb \tau^{(t)} - \nu_t \nabla_{\mb
  \tau} \varphi(\mb y \circ \mb \tau^{(t)}) \right)
  \label{eq:prox_general}
\end{equation}
from a fixed initialization $\vtau^{(0)}$, where $\nu_t>0$ is a step size sequence and
$\mr{prox}_{\rho}(\vtau) = \argmin_{\vtau'}\frac{1}{2} \norm{\vtau - \vtau'}_F^2 + \rho(\vtau')$
is well-defined if $\rho$ is a proper convex function.
Unrolled optimization suggests to truncate this iteration after $T$ steps, and treat the iterate
$\vtau^{(T)}$ at that iteration as the output of a neural network. One can then learn certain
parameters of the neural network from datasets, as a principled approach to combining the
structural priors of the original optimization problem with the benefits of a data-driven
approach. 

In \Cref{fig:motif}(e), we show an architectural diagram for a neural network unrolled from a
proximal gradient descent solver for the registration formulation \Cref{eq:formulation_reg}. 
We always initialize our networks with $\vtau^{(0)}$ as the identity transformation field,
and in this context we have $\mathrm{prox}_{\lambda \nu_t \rho}(\vtau) =
\mathrm{proj}_{\bbT}(\vtau)$ as the nearest point in $\bbT$ to $\vtau$, which can be computed
efficiently (computational details are provided in \Cref{sec:parametric}). The cost
\Cref{eq:formulation_reg} is differentiable; calculating its gradient as in \Cref{sec:grads},
\Cref{eq:prox_general} becomes
\begin{equation}
  \mb \tau^{(t+1)} = \mr{proj}_{\bbT}\left( 
    \mb \tau^{(t)} - \nu_t \sum_{k=0}^{c-1} 
    \left(\vg_{\sigma^2} \conv \sP_{\Omega}\left[
        \vg_{\sigma^2} \conv \left( \vy \circ \vtau^{(t)} - \vx_o \right)_k
    \right] \kron \One_2 \right) 
    \circleddot \left(\diff \vy_k \circ \vtau^{(t)}\right) %
  \right),
  \label{eq:registration_unrolled}
\end{equation}
as we represent visually in \Cref{fig:motif}(e), where a subscript of $k$ denotes the $k$-th
channel of the image and $\diff \vy \in \bbR^{m \times n \times c \times 2}$ is the Jacobian
matrix of $\vy$.
The constituent operations in this network
are convolutions, pointwise nonlinearities and linear maps, which lend themselves ideally to
implementation in standard deep learning software packages and on hardware accelerators; and
because the cubic convolution interpolation kernel $\phi$ is twice continuously differentiable
except at four points of $\bbR$, these networks are end-to-end differentiable and can be
backpropagated through efficiently. The calculations necessary to instantiate unrolled network
architectures for other optimization formulations used in our experiments are deferred to
\Cref{sec:grads}. A further advantage of the unrolled approach to network design is that
hyperparameter selection becomes directly connected to convergence properties of the optimization
formulation \Cref{eq:general_formulation}: we demonstrate how theory influences these selections
in \Cref{sec:theory}, and provide practical guidance for registration and detection problems
through our experiments in \Cref{sec:motif,sec:mcs_experiments,sec:basin_experiments}.

\section{Invariant Networks for Hierarchical Object Detection} \label{sec:motif}

The unrolled networks in \Cref{sec:framework} represent architectural primitives for building
deformation-invariant neural networks: they are effective at producing invariant representations
for input images containing local motifs. In this section, we illustrate how these local modules
can be combined into a network that performs invariant processing of nonlocally-structured visual
data, via an invariant hierarchical object detection task with a fixed template. %
For simplicity, in this section we will focus on the setting where $\bbT$ is the set of rigid
motions of the image plane (i.e., translations and rotations), which we will write as $\SE(2)$.

\subsection{Data Model and Problem Formulation} \label{sec:hier_data_model}

We consider an object detection task, where the objective is to locate a fixed template with
independently-articulating parts (according to a $\SE(2)$ motion model) in visual clutter. More
precisely, we assume the template is decomposable into a hierarchy of deformable parts, as in
\Cref{fig:detection}(a): at the top level of the
hierarchy is the template itself, with concrete visual motifs at the lowest levels that
correspond to specific pixel subsets of the template, which constitute independent parts. Because
these constituent parts deform independently of one another, detecting this template
\textit{efficiently} demands an approach to detection that captures the specific hierarchical
structure of the template.\footnote{Reasoning as in \Cref{sec:intro}, the effective dimension of
  the space of all observable transformations of the object is the product of the dimension of the
  motion model and the number of articulating parts. A detector that exploits the hierarchical
  structure of the object effectively reduces the dimensionality to $\dim(\mathrm{motion\,model})
  + \log(\mathrm{number\,of\,parts})$, yielding a serious advantage for moderate-dimensional
families of deformations.}
Compared to existing methods for parts-based object detection that are formulated to work with
objects subject to complicated variations in appearance
\cite{Felzenszwalb2008-ui,Felzenszwalb2010-gv,Girshick2015-qw,Pedersoli2015-wx}, focusing on the
simpler setting of template detection allows us to develop a network that guarantees invariant
detection under the motion model, and can incorporate large-scale statistical learning techniques
by virtue of its unrolled construction (although we leave this latter direction for future work).
We note that other approaches are possible, such as hierarchical sparse modeling
\cite{Jenatton2010-oj, Bar2010-iq} or learning a graphical model \cite{Sutton2012-fe}.

More formally, we write $\vy_o \in \bbR^{m_o \times n_o \times 3}$ for the RGB image corresponding
to a canonized view of the template to be detected (e.g., the crab at the top of the hierarchy in 
\Cref{fig:detection}(a) left) embedded on a black background. For a $K$-motif object (e.g., $K=4$
for the crab template), we let $\vx_k \in \bbR^{m_k \times n_k \times 3}$ denote the $k$ distinct
(canonized, black-background-embedded) transforming motifs in the object, each with
non-overlapping occurrence coordinates $(i_k, j_k) \in \set{0, \dots, m_o} \times \set{0, \dots,
n_o}$.
The template $\vy_o$ decomposes as
\begin{equation}
  \vy_o = \underbrace{\sum_{k=1}^K\vx_k \conv \ve_{i_k j_k}}_{\text{transforming motifs}} 
  + \underbrace{\vy_o - \sum_{k=1}^K \vx_k \conv \ve_{i_k j_k}}_{\text{static body}}.
  \label{eq:linked_rigid_body_1level}
\end{equation}
For example, the four transforming motifs for the crab template in \Cref{fig:detection}(a) are the
two claws and two eyes. In our experiments with the crab template, we
will consider detection of transformed templates $\vy_{\mathrm{obs}}$ of the following form:
\begin{equation}
  \vy_{\mathrm{obs}} = \left[
    \sum_{k=1}^K (\vx_k \conv \ve_{i_k j_k} )\circ \vtau_k 
    + \left(\vy_o - \sum_{k=1}^K \vx_k \conv \ve_{i_k j_k}\right)
  \right] \circ \vtau_0,
  \label{eq:linked_rigid_body_1level_obs}
\end{equation}
where $\vtau_0 \in \SE(2)$, and $\vtau_k \in \SO(2)$ is sufficiently close to the identity
transformation (which represents the physical constraints of the template). The detection task is
then to decide, given an input scene $\vy \in \bbR^{m \times n \times 3}$ containing visual
clutter (and, in practice, $m \gg m_o$ and $n \gg n_o$), whether or not a transformed instance
$\vy_{\mathrm{obs}}$ appears in $\vy$ or not, and to output estimates of its transformation
parameters $\vtau_k$.

Although our experiments will pertain to the observation model
\Cref{eq:linked_rigid_body_1level_obs}, as it agrees with our decomposition of the crab template in
\Cref{fig:detection}(a), the networks we construct in \Cref{sec:hier_networks} will be amenable to
more complex observation models where parts at intermediate levels of the hierarchy also
transform.\footnote{For example, consider a simple extension of the crab template in
  \Cref{fig:detection}(a), where the left and right claw motifs are further decomposed into two
  pairs of pincers plus the left and right arms, with opening and closing motions for the pincers,
  and the same $\SO(2)$ articulation model for the arms (which naturally moves the pincers in
accordance with the rotational motion).} To this end, we introduce additional notation that
captures the hierarchical structure of the template $\vy_o$.
Concretely, we identify a hierarchically-structured template with a directed rooted tree $G =
(V, E)$, with $0$ denoting the root node, and $1, \dots, K$ denoting the $K$ leaf nodes. Our
networks will treat observations of the form 
\begin{equation}
  \vy_{\mathrm{obs}} = 
  \sum_{k=1}^K ((\cdots(((\vx_k \conv \ve_{i_k j_k} )\circ 
  \vtau_k) \circ \vtau_{v_{d(k)-1}}) \circ \cdots ) \circ \vtau_{v_{1}}) \circ \vtau_{0}
  + 
  \left(\vy_o - \sum_{k=1}^K \vx_k \conv \ve_{i_k j_k}\right) \circ \vtau_0,
  \label{eq:linked_rigid_body_dlevel_obs}
\end{equation}
where $d(k)$ is the depth of node $k$, and $v_1, \dots, v_{d(k)-1} \in V$ with 
$0 \to v_{1} \to \dots \to v_{d(k)-1} \to k$ specifying the path from the root node to node $k$ in
$G$. To motivate the observation model \Cref{eq:linked_rigid_body_dlevel_obs}, consider the crab
example of \Cref{fig:detection}(a), where in addition we imagine the coordinate frame of the eye
pair motif transforms independently with a transformation $\vtau_5$: in this case, the observation
model \Cref{eq:linked_rigid_body_dlevel_obs} can be written in an equivalent `hierarchical' form 
\begin{equation*}
  \vy_{\mathrm{obs}} =
  \left[
    (\vx_1 \conv \ve_{i_1 j_1}) \circ \vtau_1
    + (\vx_2 \conv \ve_{i_2 j_2}) \circ \vtau_2
    + \left[
      (\vx_3 \conv \ve_{i_3 j_3}) \circ \vtau_3
      + (\vx_4 \conv \ve_{i_4 j_4}) \circ \vtau_4
    \right] \circ \vtau_5
  \right] \circ \vtau_0 + \vy_{\mathrm{body}} \circ \vtau_0,
\end{equation*}
by linearity of the interpolation operation $\vx \mapsto \vx \circ \vtau$ (with
$\vy_{\mathrm{body}} = \vy_o - \sum_k \vx_k \conv \ve_{i_k j_k}$).

\subsection{Aside: Optimization Formulations for Registration of ``Spiky'' Motifs}
\label{sec:spike_reg}

To efficiently perform hierarchical invariant detection of templates following the model
\Cref{eq:linked_rigid_body_dlevel_obs}, the networks we design will build from the following
basic paradigm, given an input scene $\vy$:
\begin{enumerate}
  \item \textbf{Visual motif detection}: First, perform detection of all of the lowest-level
    motifs $\vx_1, \dots, \vx_K$ in $\vy$.  The output of this process is an \textit{occurrence
    map} for each of the $K$ transforming motifs, i.e.\ an $m \times n$ image taking
    (ideally) value $1$ at the coordinates where detections occur and $0$ elsewhere.
  \item \textbf{Spiky motif detection for hierarchical motifs}: Detect intermediate-level
    abstractions using
    the occurrence maps in $\vy$ obtained in the previous step. For example, if $k = 3$ and $k =
    4$ index the left and right eye motifs in the crab template of \Cref{fig:detection}(a),
    detection of the eye pair motif corresponds to registration of the canonized eye pair's
    occurrence map against the two-channel image corresponding to the concatenation of $\vx_3$ and
    $\vx_4$'s occurrence maps in $\vy$.
  \item \textbf{Continue until the top of the hierarchy}: This occurrence map detection process is
    iterated until the top level of the hierarchy. For example, in \Cref{fig:detection}(a), a
    detection of the crab template occurs when the multichannel image corresponding to the
    occurrence maps for the left and right claws and the eye pair motif is matched.
\end{enumerate}

To instantiate this paradigm, we find it necessary to develop a separate registration formulation
for registration of occurrence maps, beyond the formulations we have introduced in
\Cref{sec:framework}. Indeed, occurrence maps contain no texture information and are maximally
localized, motivating a formulation that spreads out gradient information and avoids interpolation
artifacts---and although there is still a need to cope with clutter in general, the fact that the
occurrence maps are generated on a black background obviates the need for extensive background
modeling, as in \Cref{eq:formulation_withbg}. We therefore consider the following ``complementary
smoothing'' formulation for spike registration: for a $c$-channel occurrence map $\vy$ and
canonized occurrence map $\vx_o$, we optimize over the affine group $\mathrm{Aff}(2) = \GL(2)
\rtimes \bbR^2$ as
\begin{equation}
  \min_{\vA, \vb}\,
  \frac{1}{2c} \left\| \vg_{\sigma^2
    \mb I - \sigma_0^2 \vA \vA\adj} \conv \left( \mr{det}^{-1/2}( \mb A \mb A\adj ) \left(
    \vg_{\sigma_0^2 \mb I} \conv \vy \right) \circ \mb \tau_{\mb A^{-1}, -\mb \vA^{-1} \vb}
  \right) - \vg_{\sigma^2 \mb I} \conv {\mb x_o} \right\|_{F}^2 + \chi_{\mathrm{Aff}(2)}(\vA, \vb), 
  \label{eq:complementary_smoothing}
\end{equation}
where $\vg_{\vM}$ denotes a single-channel centered gaussian filter with positive definite
covariance matrix $\vM \succ 0$, and correlations are broadcast across channels.  Here, $\sigma >
0$ is the main smoothing parameter to propagate gradient information, and $\sigma_0 > 0$ is an
additional smoothing hyperparameter to mitigate interpolation artifacts.

In essence, the key modifications in \Cref{eq:complementary_smoothing} that make it amenable to
registration of occurrence maps are the compensatory effects for scaling that it introduces:
transformations that scale the image correspondingly reduce the amplitude of the (smoothed)
spikes, which is essential given the discrete, single-pixel spike images we will register.
Of course, since we consider only euclidean transformations in our experiments in this section, we
always have $\vA \vA\adj = \vI$, and the problem \Cref{eq:complementary_smoothing} can be
implemented in a simpler form.  However, these modifications lead the problem
\Cref{eq:complementary_smoothing} to work excellently for scale-changing transformations as well:
we explore the reasons behind this from both theoretical and practical perspectives in
\Cref{sec:theory}.

\begin{figure*}[!htb]
  \centering
  \includegraphics[width=\textwidth]{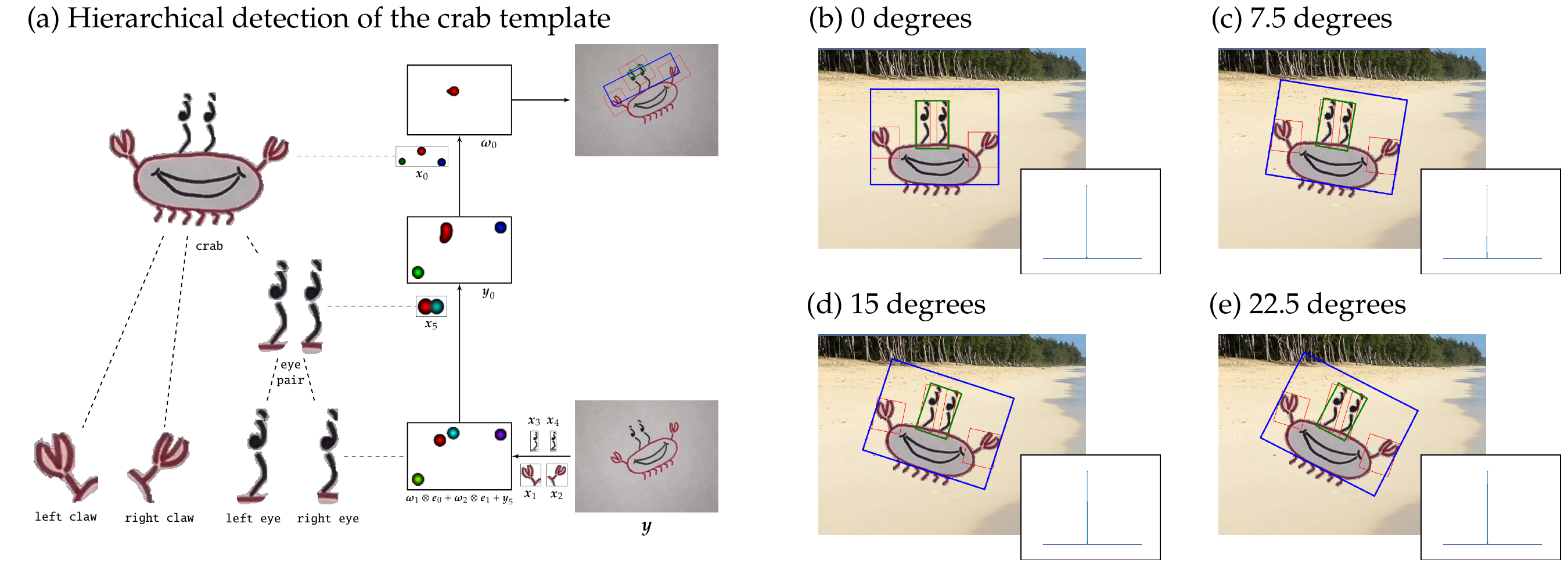}
  \caption{An example of a hierarchically-structured template, and the results of an implementation
    of our detection network. \textbf{(a)}: Structure of the crab template, described in
    \Cref{sec:hier_data_model}, and its interaction with our network architecture for detection,
    described in \Cref{sec:hier_networks}. \textit{Left: top-down decomposition of the template into
    motifs.} A template of interest $\vy_o$ (here, the crab at top left) is decomposed into
    a hierarchy of abstractions. The hierarchical structure is captured by a tree $G = (V, E)$:
    nodes represent parts or aggregations of parts, and edges represent their relationships.
    \textit{Right: bottom-up detection of the template in a novel scene.} To detect the template
    in a novel scene and pose, the network described in \Cref{sec:hier_networks} first localizes
    each of the lowest-level visual motifs at left and their transformation parameters in the
    input scene $\vy$ (bottom right). Motifs and the derived occurrence maps are labeled in
    agreement with the notation we introduce in \Cref{sec:hier_networks}.
    The output of each round of optimization is an occurrence map $\vomega_v$ for
    nodes $v \in V$; these occurrence maps then become the inputs for detection of the next level
    of concepts, following the connectivity structure of $G$, until the top-level template is
    reached (top right). \textbf{(b-e)}: Concrete results for the hierarchical invariant object
    detection network implemented in \Cref{sec:hier_results}: the learned transformation at the
    minimum-error stride for each motif is used to draw the motifs' transformed bounding boxes.
    Insets at the bottom right corner of each result panel visualize the quality of the final
    detection trace $\vomega_{0}$ for the template, with a value of $1$ at the top of the inset. %
  }
  \label{fig:detection}
\end{figure*}

\subsection{Invariant Network Architecture} \label{sec:hier_networks}

The networks we design to detect under the observation model
\Cref{eq:linked_rigid_body_dlevel_obs} consist of a configuration of unrolled motif registration
networks, as in \Cref{fig:motif}(e), arranged in a `bottom-up' hierarchical fashion, following the
hierarchical structure in the example shown in \Cref{fig:detection}(a). The configuration for each
motif registration sub-network is a `GLOM-style' \cite{Hinton2021-fm} collection of the networks
sketched
in \Cref{fig:motif}(e), oriented at different pixel locations in the input scene $\vy$; the
transformation parameters predicted of each of these configurations are aggregated across the
image, weighted by the final optimization cost (as a measure of quality of the final solution), in order to determine
detections. These detections are then used as feature maps for the next level of occurrence
motifs, which in turn undergo the same registration-detection process until reaching the top-level
object's occurrence map, which we use as a solution to the detection problem.  A suitable unrolled
implementation of the registration and detection process leads to a network that is end-to-end
differentiable and amenable to implementation on standard hardware acceleration platforms
(see \Cref{sec:hier_results}).

We now describe this construction formally, following notation introduced in
\Cref{sec:hier_data_model}. The network input is an RGB image $\vy \in \bbR^{m \times n \times
3}$.  We shall assume that the canonized template $\vy_o$ is given, as are
as the canonized visual motifs $\vx_1, \dots, \vx_K$ and their masks $\Omega_1, \dots, \Omega_K$;
we also assume that for every $v \in V$ with $v \not \in \set{1, \dots, K}$, we are given the
canonized occurrence map $\vx_v \in \bbR^{m_v \times n_v \times c_v}$ of the hierarchical feature
$v$ in $\vy_o$. In practice, one obtains these occurrence maps through a process of
``extraction'', using $\vy_o$ as an input to the network, which we describe in
\Cref{sec:hier_preprocessing}. The network construction can be separated into three distinct
steps:

\paragraph{Traversal.} The network topology is determined by a simple traversal of the graph $G$.
For each $v \in V$, let $d(v)$ denote the shortest-path distance from $0$ to $v$, with
unit weights for edges in $E$ (the ``depth'' of $v$ in $G$).  We will process motifs in a
deepest-first order for convenience, although this is not strictly necessary in all cases (e.g.\
for efficiency, it might be preferable to process all leaf nodes $1, \dots, K$ first). Let
$\diam(G) = \max_{v \in V} d(v)$, and for an integer $\ell$ no larger than $\diam(G)$, we let
$D(\ell) \in \bbN$ denote the number of nodes in $V$ that are at depth $\ell$.

\paragraph{Motif detection at one depth.} Take any integer $0 \leq \ell \leq \diam(G)$, and let
$v_1, \dots, v_{D(\ell)}$ denote the nodes in $G$ at depth $\ell$, enumerated in increasing order
(say). For each such vertex $v_k$, perform the following steps:

\subparagraph{1. Is this a leaf?} If the neighborhood $\set{v' \given (v_k, v') \in E}$ is empty,
this node is a leaf; otherwise it is not. Subsequent steps depend on this distinction. We
phrase the condition more generally, although we have defined $1, \dots, K$ as the leaf
vertices here, to facilitate some implementation-independence.

\subparagraph{2. Occurrence map aggregation for non-leaves:} If $v_k$ is not a leaf, construct
its detection feature map from lower-level detections: concretely, let
\begin{equation}
  \vy_{v_k} = \sum_{v'\,:\,(v_k, v') \in E} \vomega_{v'} \kron \ve_{\pi_{v_k}(v')},
  \label{eq:multichannel_spike_map}
\end{equation}
where $\pi_{v_k}(v')$ denotes a vertex-increasing-order enumeration of the set $\set{v'\,:\,(v_k,
v') \in E}$ starting from $0$. By construction (see the fourth step below), $\vy_{v_k}$ has
the same width and height as the input scene $\vy$, but $c_{v_k} = \abs{\set{v'\,:\,(v_k,
v') \in E}}$ channels (one for each child node) instead of $3$ RGB channels.

\subparagraph{3. Perform strided registration:} Because the motif $\vx_{v_k}$ is in general much
smaller in size than the scene $\vy_{v_k}$, and because the optimization formulation
\Cref{eq:general_formulation} is generally nonconvex with a finite-radius basin of attraction
around the true transformation parameters in the model \Cref{eq:linked_rigid_body_dlevel_obs},
the detection process consists of a search for $\vx_{v_k}$ anchored at a grid of points in
$\vy_{v_k}$. Concretely, let 
\begin{equation*}
  \Lambda_{v_k} = \set{(i\Delta_{H, v_k}, j\Delta_{W, v_k}) \given
  (i, j) \in \set{0, \dots, m-1} \times \set{0, \dots, n-1} } \cap (\set{0, \dots, m-1} \times
  \set{0, \dots, n-1})
\end{equation*}
denote the grid for the $v_k$-th motif; here $\Delta_{H, v_k}$ and
$\Delta_{W, v_k}$ define the vertical and horizontal stride lengths of the grid (we discuss
choices of these and other hyperparameters introduced below in \Cref{sec:hier_preprocessing}).
When $v_k$ is a leaf, 
for each $\vlambda \in \Lambda_{v_k}$, we let $(\vU(v_k, \vlambda), \vb(v_k, \vlambda)) \in
\SE(2)$ denote the parameters obtained after running an unrolled solver for the cost-smoothed
visual motif registration problem
\begin{equation}
  \min_{\vtau}\,
  \frac{1}{2}
  \sum_{\vDelta \in \bbZ \times \bbZ}
  (\vg_{\sigma_{v_k}^2})_{\vDelta}
  \norm*{
    \sP_{\Omega_{v_k}}\left[
      \left(
        \vg_{\sigma_{\mathrm{in}}^2} \conv \vy 
      \right)
      \circ (\vtau + \vtau_{\Zero, \vDelta + \vlambda}) - \vx_{v_k}
    \right]
  }_F^2
  + \chi_{\SE(2)}(\vtau),
  \label{eq:hier_strided_textured}
\end{equation}
for $T_{v_k}$ iterations, with step size $\nu_{v_k}$.
In addition, we employ a two-step multiscale smoothing strategy, which involves initializing
an unrolled solver for \Cref{eq:hier_strided_textured} with a much smaller smoothing parameter
$(\sigma_{v_k}')^2$ at $(\vU(v_k, \vlambda), \vb(v_k, \vlambda))$ and running it for an
additional fixed number of iterations; we let
$\mathrm{loss}(v_k, \vlambda)$ denote the final objective function value after this multiscale
process, and abusing notation, we let $(\vU(v_k, \vlambda), \vb(v_k, \vlambda))$ denote the
updated final parameters . Precise implementation details are discussed in
\Cref{sec:hier_preprocessing}.  When $v_k$ is not a leaf, we instead define the same fields on
the grid $\Lambda_{v_k}$ via a solver for the spike registration problem
\begin{equation}
  \min_{\vtau}\,
  \frac{1}{2c_{v_k}} 
  \left\| \sP_{\Omega_{v_k}} \left[
    \vg_{\sigma_{v_k}^2 - \sigma_{0, v_k}^2} \conv 
    \left(
      \vy_{v_k} \circ (\vtau + \vtau_{\Zero, \vlambda})
      - \vx_{v_k}
    \right)
  \right]
  \right\|_{F}^2
  + \chi_{\SE(2)}(\vtau), 
  \label{eq:hier_strided_spiky}
\end{equation}
with $\Omega_{v_k}$ denoting a dilated bounding box for $\vx_{v_k}$, and otherwise the same
notation and hyperparameters. We do not use multiscale smoothing for non-leaf motifs.

\subparagraph{4. Aggregate registration outputs into detections (occurrence maps):} We convert the
registration fields into detection maps, by computing
\begin{equation}
  \vomega_{v_k}
  =
  \sum_{\vlambda \in \Lambda_{v_k}}
  \left(
    \vg_{\sigma_{0, v_k}^2} \conv \ve_{\vlambda + \vb(v_k, \vlambda)}
  \right)
  \exp\Bigl( 
    - \alpha_{v_k} \max \set*{0, \mathrm{loss}(v_k, \vlambda) - \gamma_{v_k}}
  \Bigr),
  \label{eq:hier_strided_occmaps}
\end{equation}
where each summand $\vg_{\sigma_{0, v_k}^2} \conv \ve_{\vlambda + \vb(v_k, \vlambda)}$ is
truncated to be size $m \times n$.\footnote{This convolutional notation is of course an abuse
  of notation, to avoid having to define a gaussian filter with a general mean parameter. In
  practice, this latter technique is both more efficient to implement and leads to a
stably-differentiable occurrence map.}
The scale and threshold parameters $\alpha_{v_k}$ and $\gamma_{v_k}$ appearing in this formula
are calibrated to achieve a specified level of performance under an assumed maximum level of
visual clutter and transformation for the observations \Cref{eq:linked_rigid_body_dlevel_obs},
as discussed in \Cref{sec:hier_preprocessing}. 

We prefer to embed detections as occurrence maps and use these as inputs for higher-level
detections using optimization, rather than a possible alternate approach (e.g.\ extracting
landmarks and processing these using group synchronization), in order to have each occurrence map
$\vomega_{v}$ for $v \in V$ be differentiable with respect to the various filters and
hyperparameters.

\paragraph{Template detection.} To perform detection given an input $\vy$, we repeat the four
steps in the previous section for each motif depth, starting from depth $\ell=\diam(G)$, and each
motif at each depth. After processing depth $\ell=0$, the output occurrence map $\vomega_0$ can be
thresholded to achieve a desired level of detection performance for observations of the form
\Cref{eq:linked_rigid_body_dlevel_obs}. The detection process is summarized as
\Cref{alg:hier_detection}.

By construction, this output $\vomega_0$ can be
differentiated with respect to each node $v \in V$'s hyperparameters or filters, and the unrolled
structure of the sub-networks and $G$'s topology can be used to efficiently calculate such
gradients via backpropagation. In addition, although we do not use the full transformation
parameters $\vU(\vlambda, v)$ calculated in the registration operations
\Cref{eq:hier_strided_textured,eq:hier_strided_spiky}, these can be leveraged for various purposes
(e.g.\ drawing detection bounding boxes, as in our experimental evaluations in
\Cref{sec:hier_results}).

\begin{algorithm}[tb]
  \caption{Invariant Hierarchical Motif Detection Network, Summarizing \Cref{sec:hier_networks}}
  \begin{algorithmic}
    \INPUT scene $\vy$, graph $G = (V, E)$, motifs $(\vx_v, \Omega_v)_{v\in V}$, hyperparameters
    $(\nu_v, T_v, \Delta_{H, v}, \Delta_{W, v}, \sigma_v^2, \sigma_{0,v}^2, \alpha_v, \gamma_v
    )_{v \in V}$
    \STATE set $\diam(G)$ and node enumerations %
    by depth-first traversal of $G$
    \FORALL{depths $\ell = \diam(G), \diam(G) - 1, \dots, 0$}
    \FORALL{nodes $v$ at depth $\ell$}
    \STATE set $N_v = \set{v' \given (v, v') \in E}$ and $c_v = \abs{N_v}$ 
    \IF{$c_v > 0$}
    \STATE concatenate occurrence maps into $\vy_{v} = \sum_{v' \in N_v} \vomega_{v'} \kron
    \ve_{\pi_{v}(v')}$
    \ENDIF
    \FORALL{$\vlambda \in \Lambda_v(\Delta_{H, v}, \Delta_{W, v})$}
    \IF{$c_v > 0$}
    \STATE set $\vU(v, \vlambda), \vb(v, \vlambda)  = \argmin_{\vtau}
    \frac{1}{2c_{v}} 
    \norm{ \vg_{\sigma_{v}^2 - \sigma_{0, v}^2} \conv 
      ( \vy_{v} \circ (\vtau + \vtau_{\Zero, \vlambda}) - \vx_{v})
    }_{F}^2
    + \chi_{\SE(2)}(\vtau)$ 
    \STATE set $\mathrm{loss}(v, \vlambda) = \min_{\vtau}
    \frac{1}{2c_{v}} 
    \norm{ \vg_{\sigma_{v}^2 - \sigma_{0, v}^2} \conv 
      ( \vy_{v} \circ (\vtau + \vtau_{\Zero, \vlambda}) - \vx_{v})
    }_{F}^2
    + \chi_{\SE(2)}(\vtau)$ 
    \STATE (both with a $T_v$-layer unrolled solver)
    \ELSE
    \STATE set $\vU(v, \vlambda), \vb(v, \vlambda) = \argmin_{\vtau} \tfrac{1}{2}
    \sum_{\vDelta}
    (\vg_{\sigma_{v}^2})_{\vDelta}
    \norm{
      \sP_{\Omega_{v}}[ ( \vg_{\sigma_{\mathrm{in}}^2} \conv \vy ) \circ (\vtau + \vtau_{\Zero,
      \vDelta + \vlambda}) - \vx_{v} ]
    }_F^2
    + \chi_{\SE(2)}(\vtau)$
    \STATE set $\mathrm{loss}(v, \vlambda) = \min_{\vtau} \tfrac{1}{2}
    \sum_{\vDelta}
    (\vg_{\sigma_{v}^2})_{\vDelta}
    \norm{
      \sP_{\Omega_{v}}[ ( \vg_{\sigma_{\mathrm{in}}^2} \conv \vy ) \circ (\vtau + \vtau_{\Zero,
      \vDelta + \vlambda}) - \vx_{v} ]
    }_F^2
    + \chi_{\SE(2)}(\vtau)$
    \STATE (both with a $T_v$-layer unrolled solver, with two-round multiscale smoothing)
    \ENDIF
    \ENDFOR{}
    \STATE construct the occurrence map $ \vomega_{v} = \sum_{\vlambda \in \Lambda_{v}}
    ( \vg_{\sigma_{0, v}^2} \conv \ve_{\vlambda + \vb(v, \vlambda)})
    \exp( - \alpha_{v} \max \set{0, \mathrm{loss}(v, \vlambda) - \gamma_{v}}) $
    \ENDFOR{}
    \ENDFOR{}
    \OUTPUT template occurrence map $\vomega_0$
  \end{algorithmic}
  \label{alg:hier_detection}
\end{algorithm}

\subsection{Implementation and Evaluation} \label{sec:hier_results}
We implement the hierarchical invariant object detection
network described in \Cref{sec:hier_networks} in PyTorch \cite{NEURIPS2019_9015}, and test it for
detection of the crab template from \Cref{fig:detection}(a) subject to a global rotation (i.e.,
$\vtau_0$ in the model \Cref{eq:linked_rigid_body_dlevel_obs}) of varying size
(\Cref{fig:detection}(b-e)). 
In $512 \times 384$ pixel scenes on a ``beach'' background, a
calibrated detector perfectly detects the crab from its constituent parts up to rotations of
$\pi/8$ radians---at rotations around $\pi/6$, a multiple-instance issue due to similarity between
the two eye motifs begins to hinder the detection performance.  Traces in each panel of
\Cref{fig:detection}, right demonstrate the precise localization of the template.

For hyperparameters, we set $T_v = 1024$ and $\Delta_{H,v} = \Delta_{W,v} = 20$ for all $v \in V$,
and calibrate detection parameters as described in \Cref{sec:hier_preprocessing};
for visual motifs, we calibrate the remaining registration hyperparameters as described in
\Cref{sec:hier_preprocessing} on a per-motif basis, and for spike motifs, we find the
prescriptions for $\sigma_v^2$ and the step sizes $\nu_v$ implied by theory (\Cref{sec:theory}) to
work excellently without any fine-tuning.
We also implement selective filtering of strides for spiky motif alignment that are unlikely to
succeed: due to the common background, this type of screening is particularly effective here.
The strided registration formulations \Cref{eq:hier_strided_textured,eq:hier_strided_spiky} afford
efficient batched implementation on a hardware accelerator, given that the motifs $\vx_v$ for $v
\in V$ are significantly smaller than the full input scene $\vy$, and the costs only depend on
pixels near to the motifs $\vx_v$.
On a single NVIDIA TITAN X Pascal GPU accelerator (12 GB memory), it takes approximately
five minutes to complete a full detection. We expect throughput to be further improvable without
sacrificing detection performance by decreasing the maximum iterations for each unrolled network
$T_v$ even further---the setting of $T_v = 1024$ is conservative, with convergence typically much
more rapid. Our implementation is available at \codeurl.

\section{Guaranteed, Efficient Detection of Occurrence Maps} \label{sec:theory}
In \Cref{sec:motif}, we described how invariant processing of hierarchically-structured visual
data naturally leads to problems of registering `spiky' occurrence maps, and we introduced the
formulation \Cref{eq:complementary_smoothing} for this purpose. In this section, we provide a
theoretical analysis of a continuum model for the proximal gradient descent method applied to the
optimization formula \Cref{eq:complementary_smoothing}.
A byproduct of our analysis is a concrete prescription for the step size and rate of
smoothing---in \Cref{sec:mcs_experiments}, we demonstrate experimentally that these prescriptions
work excellently for the discrete formulation \Cref{eq:complementary_smoothing}, leading to rapid
registration of the input scene.

\begin{figure*}[!htb]
  \centering
  \includegraphics[width=0.9\linewidth]{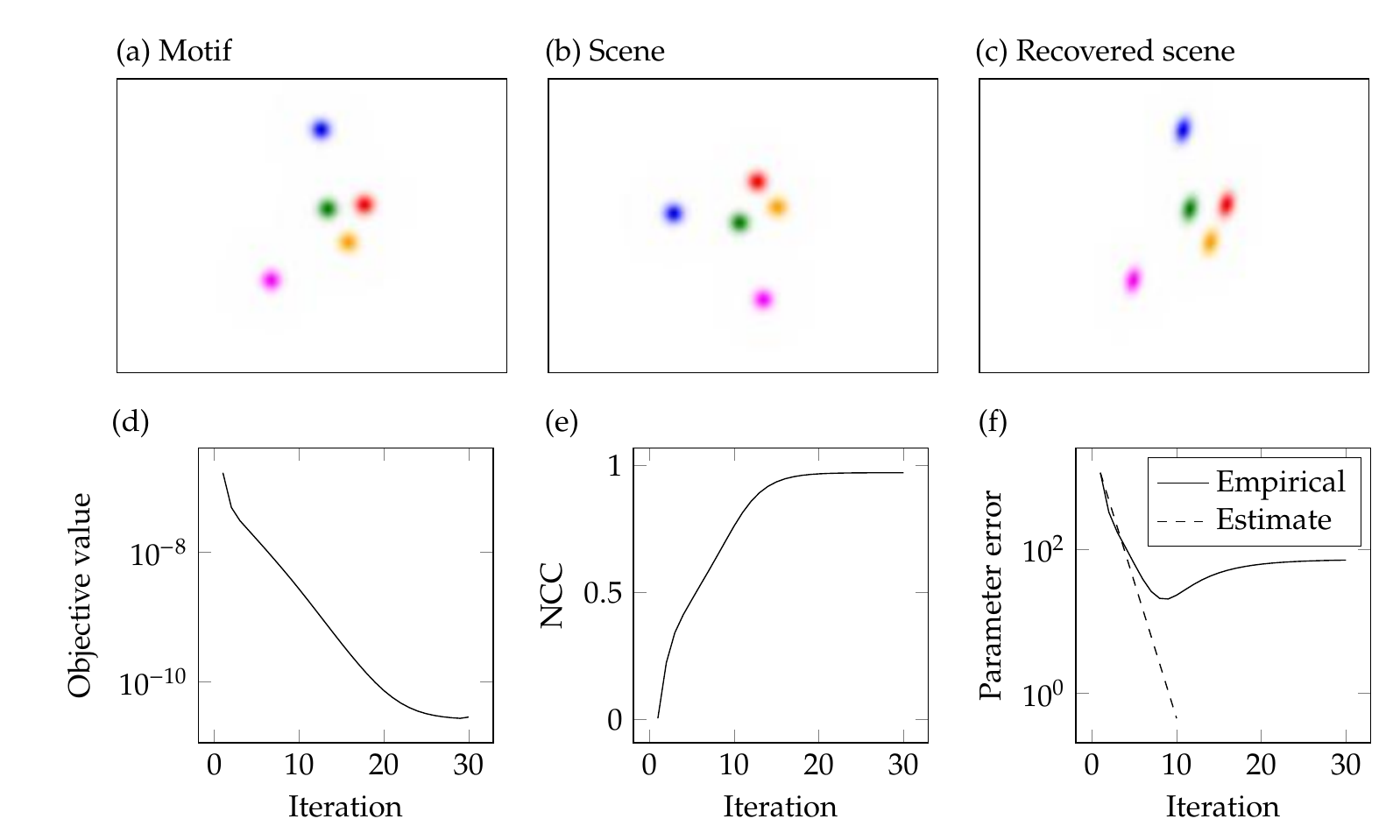}
  \caption{Numerical verification of \Cref{thm:multichannel-spike}. \textbf{(a):} A multichannel
    spike motif containing 5 spikes. \textbf{(b):} A scene generated by applying a random affine
    transformation to the motif. \textbf{(c):} The solution to \Cref{eq:complementary_smoothing}
    with these data. The skewing apparent here is undone by the compensated external gaussian
    filter, which enables accurate localization in spite of these artifacts.  \textbf{(d):} Change
    in objective value of \Cref{eq:complementary_smoothing} across iterations of proximal gradient
    descent. Convergence occurs in tens of iterations. \textbf{(e):} Change in normalized cross
    correlation across iterations (see \Cref{sec:extra_experiment_details}). We observe that the
    method successfully registers the multichannel spike scene. \textbf{(f):} Comparison between
    the left and right-hand side of equation \Cref{eq:opt-main-boundmain} with gradient descent
    iterates from \Cref{eq:complementary_smoothing} (labeled as $\varphi$ here). After an initial
    faster-than-predicted linear rate, the discretized solver saturates at a sub-optimal level.
    This is because because accurate estimation of the transformation parameters $(\mb A, \mb b)$
    requires subpixel-level preciseness, which is affected by discretization and interpolation
    artifacts. It does not hinder correct localization of the scene, as (e) shows.}
  \label{fig:convergence}
\end{figure*}

\subsection{Multichannel Spike Model} We consider continuous signals defined on $\bbR^2$ in this
section, as an `infinite resolution' idealization of discrete images, free of interpolation
artifacts. We refer to \Cref{app:mc-spike-affine} for full technical details. Consider a target
signal  
\begin{equation*}%
    \mb X_o = \sum_{i = 1}^c \mb \delta_{\mb v_i} \otimes  \mb e_i, 
\end{equation*}
where $\vdelta_{\mb \vv_i}$ is a Dirac distribution centered at the point $\vv_i$,
and an observation  
\begin{equation*}%
    \mb X = \sum_{i=1}^c  \mb \delta_{\mb u_i} \otimes \mb e_i,
\end{equation*}
satisfying 
\begin{equation*}%
    \mb v_i = \mb A_\star \mb u_i + \mb b_\star.
\end{equation*}
In words, the observed signal is an affine transformation of the spike signal $\vX_o$, as in
\Cref{fig:convergence}(a, b). This model is directly motivated by the occurrence maps
\Cref{eq:multichannel_spike_map} arising in our hierarchical detection networks.
Following \Cref{eq:complementary_smoothing}, consider the objective function
\begin{equation*}
\varphi_{L^2,\sigma}(\mb A, \mb b) \equiv \frac{1}{2c} \sum_{i=1}^c \left\| \mb g_{\sigma^2 \mb
I - \sigma_0^2 (\mb A^* \mb A)^{-1}} \conv \left( \mr{det}^{1/2}( \mb A^* \mb A )
\left( \mb g_{\sigma_0^2 \mb I} \conv \mb X_{i} \right) \circ \mb \tau_{\mb A, \mb b} \right) -
\mb g_{\mb 0,\sigma^2 \mb I} \conv {(\mb X_o)}_{i} \right\|_{L^2}^2.
\end{equation*}
We study the following ``inverse parameterization'' of this function: 
\begin{equation}
 \varphi^{\mr{inv}}_{L^2,\sigma}(\mb A, \mb b) \equiv \varphi_{L^2,\sigma}( \mb A^{-1}, - \mb A^{-1} \mb b ).
 \label{eq:mcs_inverse_param}
\end{equation}
We analyze the performance of gradient descent for solving the optimization problem 
\begin{equation*}
        \min_{\mb A, \mb b} \varphi^{\mr{inv}}_{L^2,\sigma}(\mb A, \mb b).
\end{equation*}
Under mild conditions, local minimizers of this problem are global. Moreover, if $\sigma$ is set appropriately, the method exhibits linear convergence to the truth:

\begin{theorem}[Multichannel Spike Model, Affine Transforms, $L^2$] \label{thm:multichannel-spike} 
  Consider an instance of the multichannel spike model, with
  $\mb U = [\mb u_1, \dots, \mb u_c ] \in \bb R^{2 \times c}$. Assume that the spikes $\mb U$ are
  centered and nondegenerate, so that $\mb U \One = \mb 0$ and $\rank(\vU) = 2$. Then
  gradient descent
  \begin{align*}
    \mb A_{k+1} &= \mb A_k - t_{\mb A} \nabla_{\mb A} \varphi^{\mr{inv}}_{L^2,\sigma}(\mb
    A_k, \mb b_k ), \\ 
    \mb b_{k+1} &= \mb b_k - t_{\mb b} \nabla_{\mb b} \varphi^{\mr{inv}}_{L^2,\sigma}(\mb
    A_k, \mb b_k ) 
  \end{align*}
  with smoothing 
  \begin{equation}
    \sigma^2 
    \ge 2\frac{\max_i\, \norm{\vu_i}_2^2}{s_{\min}(\vU)^2} %
    \left(  
      {s_{\max}(\vU)^2 \| \mb A_\star - \mb I \|_F^2}%
      + c{  \| \mb b_\star \|_2^2 }%
    \right)
    \label{eq:sigma-bound}
  \end{equation}
  and step sizes 
  \begin{equation}
    \begin{split}
      t_{\mb A} &= \frac{8\pi c\sigma^4}{ s_{\max}(\vU)^2 }, \\
      t_{\mb b} &= 8\pi \sigma^4,
    \end{split}
    \label{eq:spike_reg_stepsizes}
  \end{equation}
  from initialization $\mb A_0 = \mb I, \mb b_0 = \mb 0$ satisfies 
  \begin{align}
    \frac{8\pi \sigma^4}{t_{\mb A}} 
    \| \mb A_k - \mb A_\star \|_{F}^2 + \| \mb b_k - \mb b_\star \|_2^2 &\le
    \left( 1 - \frac{1}{2 \kappa} \right)^{2k} \Bigl( \frac{8\pi \sigma^4}{t_{\mb A}} \| \mb I -
    \mb A_\star \|_{F}^2 + \| \mb b_\star \|_{2}^2 \Bigr), \label{eq:opt-main-boundmain}
  \end{align}
  where 
  \begin{equation*}
    \kappa = \frac{ s_{\max}( \mb U)^2 }{ s_{\min}( \mb U)^2 },
  \end{equation*}
  with, $s_{\min}(\vU)$ and $s_{\max}(\vU)$ denoting the minimum and maximum
  singular values of the matrix $\vU$. 
\end{theorem}

\Cref{thm:multichannel-spike} establishes a global linear rate of convergence for the continuum
occurrence map registration formulation \Cref{eq:mcs_inverse_param} in the relevant product
norm, where the rate depends on the condition number of the matrix of observed spike locations
$\vU$. This dependence arises from the intuitive fact that \textit{recovery} of the transformation
parameters $(\vA_{\star}, \vb_{\star})$ is a more challenging problem than registering the
observation to the motif---in practice, we do not observe significant degradation of the ability to
rapidly register the observed scene as the condition number increases.  The proof of
\Cref{thm:multichannel-spike} reveals that the use of inverse parameterization in
\Cref{eq:mcs_inverse_param} dramatically improves the landscape of optimization: the problem
becomes strongly convex around the true parameters when the smoothing level is set appropriately.
In particular, \Cref{eq:sigma-bound} suggests a level of smoothing commensurate with the maximum
distance the spikes need to travel for a successful registration, and
\Cref{eq:spike_reg_stepsizes} suggests larger step sizes for larger smoothing levels, with
appropriate scaling of the step size on the $\vA$ parameters to account for the larger motions
experienced by objects further from the origin. In the proof, the `centered locations'
assumption $\vU \One = \Zero$ allows us to obtain a global linear rate of convergence in both the
$\vA$ and $\vb$ parameters. This is not a restrictive assumption, as in practice it is always
possible to center the spike scene (e.g., by computing its center of mass and subtracting), and we
also find it to accelerate convergence empirically when it is applied.

\subsection{Experimental Verification} \label{sec:mcs_experiments}
To verify the practical implications of \Cref{thm:multichannel-spike}, which is formulated in the
continuum, we conduct numerical experiments on registering affine-transformed multichannel spike
images using the discrete formulation \Cref{eq:complementary_smoothing}. We implement a proximal
gradient descent solver for \Cref{eq:complementary_smoothing}, and use it to register
randomly-transformed occurrence maps, as visualized in \Cref{fig:convergence}(a-b). We set the
step sizes and level of smoothing in accordance with \Cref{eq:sigma-bound,eq:spike_reg_stepsizes},
with a complementary smoothing value of $\sigma_0 = 3$.  \Cref{fig:convergence} shows
representative results taken from one such run: the objective value rapidly converges to near
working precision, and the normalized cross-correlation between the transformed scene and the
motif rapidly reaches a value of $0.972$.  This rapid convergence implies the formulation
\Cref{eq:complementary_smoothing} is a suitable base for an unrolled architecture with mild depth,
and is a direct consequence of the robust step size prescription offered by
\Cref{thm:multichannel-spike}.  \Cref{fig:convergence}(f) plots the left-hand and right-hand sides
of the parameter error bound \Cref{eq:opt-main-boundmain} to evaluate its applicability to the
discretized formulation: we observe an initial faster-than-predicted linear rate, followed by
saturation at a suboptimal value. This gap is due to the difference between the continuum theory
of \Cref{thm:multichannel-spike} and practice: in the discretized setting, interpolation errors
and finite-resolution artifacts prevent subpixel-perfect registration of the parameters, and hence
exact recovery of the transformation $(\vA_{\star}, \vb_{\star})$.  In practice, successful
registration of the spike scene, as demonstrated by \Cref{fig:convergence}(e), is sufficient for
applications, as in the networks we develop for hierarchical detection in \Cref{sec:motif}.

\section{Basin of Attraction for Textured Motif Registration with
\Cref{eq:formulation_withbg}} \label{sec:basin_experiments}
The theory and experiments we have presented in \Cref{sec:theory} justify the use of local
optimization for alignment of spiky motifs. In this section, we provide additional corroboration
beyond the experiment of \Cref{fig:centerpiece} of the efficacy of our textured motif registration
formulation \Cref{eq:formulation_withbg}, under euclidean and similarity motion models. 
To this end, in
\Cref{fig:basin} we empirically evaluate the basin of attraction of a suitably-configured solver
for registration of the crab body motif from \Cref{fig:motif} with this formulation. 
Two-dimensional search grids are generated for each of the two setups as shown in the figure. For
each given pair of transformation parameters, a similar multi-scale scheme over $\sigma$ as in the
above complexity experiment is used, starting at $\sigma=10$ and step size 0.05, and halved every
50 iterations. The process terminates after a total of 250 iterations. The final ZNCC calculated
over the motif support is reported, and the figure plots the average over 10 independent runs,
where the background image is randomly generated for each pair of parameters in each run. The ZNCC
ranges from $0$ to $1$, with a value of $1$ implying equality of the channel-mean-subtracted motif
and transformed image content over the corresponding support (up to scale).

\begin{figure}[!htb]
    \centering
    \includegraphics[width=.8\linewidth]{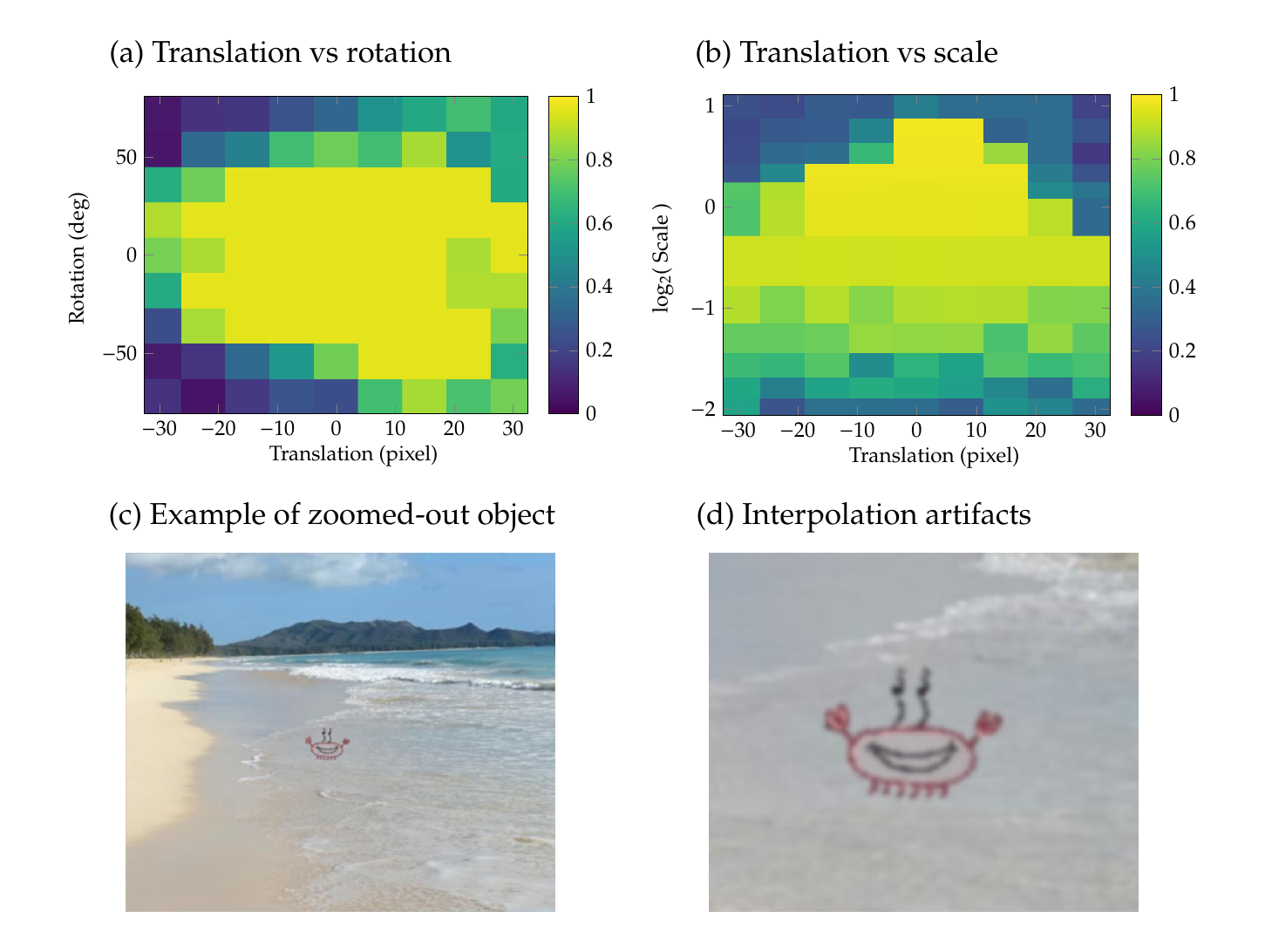}
    \caption{Plotting a basin of attraction for the textured motif registration formulation
      \Cref{eq:formulation_withbg}. \textbf{(a):} Heatmap of the ZNCC at convergence (see
      \Cref{sec:extra_experiment_details}), for translation versus rotation. Optimization
      conducted with $\SE(2)$ motion model.  \textbf{(b):} Heatmap of the ZNCC at convergence, for
      translation versus scale. Optimization conducted in `similarity mode', a $\SE(2)$ motion
      model with an extra global scale parameter. In both experiments, each
      reported data point is averaged over 10 independent runs. \textbf{(c-d):} Notably, when the
      registration target $\vy$ is zoomed out relative to the motif $\vx_o$, resolution is lost in
      the detection target, so recovering it will cause interpolation artifacts and blur the
      image. This prevents the ZNCC value from converging to 1 despite correct alignment with the
    motif, and accounts for the results shown in (b) at small scales. }
    \label{fig:basin}
\end{figure}

Panels (a) and (b) of \Cref{fig:basin} show that the optimization method tends to succeed
unconditionally up to moderate amounts of transformation.  
For larger sets of transformations, it is important to first appropriately center the image, which
will significantly improve the optimization performance.  In practice, one may use a combination
of optimization and a small number of covering, so that the entire transformation space is covered
by the union of the basins of attractions.
We note that irregularity near the
edges, especially in panel (a), can be attributed in part due to the randomness in the background
embedding, and in this sense the size of the basin in these results conveys a level of performance
across a range of simulated operating conditions. In general, these basins are also
motif-dependent: we would expect these results to change if we were testing with the eye motif
from \Cref{fig:detection}(a), for example.
A notable phenomenon in \Cref{fig:basin}(b), where translation is varied against scale, is the
lack of a clear-cut boundary of the basin at small scales. This is due to the effect illustrated
in \Cref{fig:basin}(c-d), where interpolation artifacts corrupt the motif when it is `zoomed out' by
optimization over deformations, and hence registration can never achieve a ZNCC close to 1. 
For applications where perfect reconstruction is not required, such as the hierarchical detection
task studied in \Cref{sec:motif}, these interpolation artifacts will not hinder the ability to
localize the motif in the scene at intermediate scales, and if the basin were generated with a
success metric other than ZNCC, a better-defined boundary to the basin would emerge.

\section{Discussion}
In this paper, we have taken initial steps towards realizing the potential of optimization over
transformations of domain as an approach to achieve resource-efficient invariance with visual
data. Below, we discuss several important future directions for the basic framework we have
developed.

\paragraph{Statistical variability and complex tasks.} To build invariant networks for complex
visual tasks and real-world data beyond matching against fixed templates $\vx_o$, it will be
necessary to incorporate more refined appearance models for objects, such as a sparse dictionary
model or a deep generative model \cite{Song2021-gj,Brock2018-ya,Dai2021-rg}, and train the
resulting hybrid networks in an end-to-end fashion. The invariant architectures we have designed
in this work naturally plug into such a framework, and will allow for investigations similar to
what we have developed in \Cref{sec:motif} into challenging tasks with additional structure (e.g.,
temporal or 3D data).  Coping with the more complex motion models in these applications will
demand regularizers $\rho$ for our general optimization formulation \Cref{eq:general_formulation}
that go beyond parametric constraints.

\paragraph{Theory for registration of textured motifs in visual clutter.} Our experiments in
\Cref{sec:theory} have demonstrated the value that theoretical studies of optimization
formulations have with respect to the design of the corresponding unrolled networks. Extending our
theory for spiky motif registration to more general textured motifs will enable similar insights
into the roles played by the various problem parameters in a formulation like
\Cref{eq:formulation_withbg} with respect to texture and shape properties of visual data and the
clutter present, and allow for similarly resource-efficient architectures to be derived in
applications like the hierarchical template detection task we have developed in
\Cref{sec:hier_networks}.

\paragraph{Hierarchical detection networks in real-time.} The above directions will enable the
networks we have demoed for hierarchical detection in \Cref{sec:motif} to scale to more general
data models. At the same time, there are promising directions to improve the efficiency of the
networks we design for a task like this one at the modeling level.  For example, the networks we
design in \Cref{sec:hier_networks} essentially operate in a `sliding window' fashion, without
sharing information across the strides $\vlambda$, and they perform registration and detection
separately. An architecture developed around an integrated approach to registration and detection,
possibly building off advances in convolutional sparse modeling
\cite{Kuo2019-as,Lau2019-kq,Qu2019-jf}, may lead to further efficiency gains and push the
resulting network closer to real-time operation capability.

\section*{Acknowledgments}
This work was supported by the National Science Foundation through grants NSF 1733857, NSF
1838061, NSF 1740833, and NSF 2112085,
and by a fellowship award (SB) through the National Defense Science and Engineering Graduate
(NDSEG) Fellowship Program, sponsored by the Air Force Research Laboratory (AFRL), the Office of
Naval Research (ONR) and the Army Research Office (ARO).
The authors thank Mariam Avagyan, Ben Haeffele, and Yi Ma for helpful discussions and feedback.

\printbibliography

\newpage
\appendix
\onecolumn

\section{Implementation and Experimental Details}
\subsection{Implementation Details for Parametric Transformations of the Image Plane}
\label{sec:parametric} 

Our implementation of parametric image deformations revolves around the specific definition of
interpolation we have made:
\begin{equation*}
  \vy \circ \vtau = 
  \sum_{(k, l) \in \bbZ^2} y_{kl} \phi(\vtau_0 - k\One) \circleddot \phi(\vtau_1 - l\One),
\end{equation*}
and the identification of the image $\vy \in \bbR^{m \times n}$ with a function on $\bbZ^2$ with
support in $\set{0, \dots, m-1} \times \set{0, \dots, n-1}$.\footnote{These conventions are not
  universal, although they seem most natural from a mathematical standpoint---for example, PyTorch
thinks of its images as lying on a grid in the square $[-1, +1]\times [-1, +1]$ instead, with
spacing and offsets depending on the image resolution and other implementation-specific options.
In our released code, we handle conversion from our notation to this notation.} Although we use
the notation $\circ$ for interpolation in analogy with the usual notation for composition of
functions, this operation is significantly less well-structured: although we \textit{can} define
interpolation of motion fields $\vtau_0 \circ \vtau_1$, it is impossible in general to even have
associativity of $\circ$ (let alone inverses), so that in general $(\vx \circ \vtau_0) \circ
\vtau_1 \neq \vx \circ (\vtau_0
\circ \vtau_1)$. This failure is intimately linked to the existence of parasitic interpolation
artifacts when computing and optimizing with interpolated images, which we go to great lengths to
avoid in our experiments. On
the other hand, there does exist a well-defined identity vector field:
from our definitions, we can read off the canonical definition of the identity transformation,
from which definitions for other parametric transformations we consider here follow. Defining
(with a slight abuse of notation)
\begin{equation*}
  \vm = \begin{bmatrix}
    0 \\
    1 \\
    \vdots \\
    m-1
  \end{bmatrix};
  \quad
  \vn = \begin{bmatrix}
    0 \\
    1 \\
    \vdots \\
    n-1
  \end{bmatrix},
\end{equation*}
we have from the definition of the cubic convolution interpolation kernel $\phi$ that
\begin{equation*}
  \vy \circ \left( \vm \One\adj \kron \ve_0 
  + \One \vn\adj \kron \ve_1  \right)
  = \vy.
\end{equation*}
One can then check that the following linear embedding of the affine transformations, which we
will write as $\mathrm{Aff}(2) = \GL(2) \times \bbR^2$, leads to the natural vector field
analogue of affine transformations on the continuum $\bbR^2$ (c.f.\ \Cref{app:mc-spike-affine}):
\begin{equation}
  \mathrm{Aff}(2) \cong
    \Span \set*{
      \vm \One\adj \kron \ve_0,
      \One \vn\adj \kron \ve_0,
      \vm \One\adj \kron \ve_1,
      \One \vn\adj \kron \ve_1,
      \One_{m,n} \kron \ve_0,
      \One_{m,n} \kron \ve_1
    }.
    \label{eq:affine_subspace}
\end{equation}
Of course, these vector fields can be any size---they need not match the size of the image.
As we mention in \Cref{sec:unrolling}, we always initialize our networks with the identity
transform; in the basis above, this corresponds to the vector $(1, 0, 0, 1, 0, 0)$ (i.e., this is
like a row-wise flattening of the affine transform's matrix $\vA \in \GL(2)$, concatenated with
$\vb$). 

Next we turn to computation of the proximal operator, which we need for unrolling (see
\Cref{sec:unrolling}). Given \Cref{eq:registration_unrolled} and the fact that
\Cref{eq:affine_subspace} is a subspace, we can compute the proximal operator for
$\mathrm{Aff}(2)$ given an orthonormal basis for $\mathrm{Aff}(2)$. It is then unfortunate that
the natural basis vectors that we have used in the expression \Cref{eq:affine_subspace} are not
orthogonal: we have $\ip{\vm \One\adj}{\One \vn\adj} = \ip{\vm}{\One} \ip{\vn}{\One} \gg 0$, for
example. To get around this, in practice we apply a technique we refer to as \textit{centering} of
transformations. Indeed, notice that for any $\vc \in \bbR^2$, we have
\begin{equation}
  \begin{split}
    \mathrm{Aff}(2) \cong
    &\Span \set*{
      (\vm - c_0 \One) \One\adj \kron \ve_0,
      \One (\vn - c_1 \One)\adj \kron \ve_0,
      (\vm - c_0 \One) \One\adj \kron \ve_1,
      \One (\vn - c_1 \One)\adj \kron \ve_1,
      \One_{m,n} \kron \ve_0,
      \One_{m,n} \kron \ve_1
    } \\
    &\quad+ \One_{m, n} \kron \vc.
  \end{split}
  \label{eq:affine_subspace_centered}
\end{equation}
In the continuum, applying an affine transform in this way corresponds to the mapping $\vx \mapsto
\vA(\vx - \vc) + \vb + \vc$, hence the name: the image plane is shifted to have its origin at
$\vc$ for the purposes of applying the transform. When we implement affine transforms as
suggested by \Cref{eq:affine_subspace_centered}, we choose $\vc$ to make the basis vectors
orthogonal this necessitates that $\vc = ((m-1)/2, (n-1)/2)$. Then we are able to write down a
concrete expression for the projection operator in these coordinates:\footnote{In practice, our
  choice of step size is made to scale each element in this basis to be orthonormal (in
  particular, applying different steps to the matrix and translation parameters of the
  transformation)---strictly speaking the projection in \Cref{eq:projection_onto_aff} is not the
  orthogonal projection because this extra scaling has not been applied. We do not specify this
  scaling here because its optimal value often depends on the image content: for example, see the
step size prescriptions in \Cref{thm:multichannel-spike}.}
\begin{equation}
  \mathrm{proj}_{\mathrm{Aff}(2)}(\vtau)
  =
  \left(
    (\vm - \tfrac{m-1}{2} \One)\adj \vtau_0 \One,
    \One\adj\vtau_0(\vn - \tfrac{n-1}{2} \One),
    (\vm - \tfrac{m-1}{2} \One)\adj \vtau_1 \One,
    \One\adj\vtau_1(\vn - \tfrac{n-1}{2} \One),
    \One\adj\vtau_0\One,
    \One\adj\vtau_1\One
  \right).
  \label{eq:projection_onto_aff}
\end{equation}
The low-rank structure of the basis vectors implies that this transformation can be computed quite
rapidly.
Although it may seem we have undertaken this discussion for the sake of mathematical rigor, in our
experiments we observe significant computational benefits to centering by the prescription above.
For example, when computing with \Cref{eq:complementary_smoothing}, using a non-orthogonal basis
for the affine transforms (or a center that is not at the center of the region being transformed)
often leads to skewing artifacts in the final transformation recovered.  We also notice slower
convergence.

Finally, for our experiments in \Cref{sec:motif} with the rigid motion model $\SE(2)$, some
additional discussion is required. This is because the orthogonal transformations $\SO(2)$ are not
a linear subspace, like the affine transforms \Cref{eq:affine_subspace}, but a smooth manifold
(diffeomorphic to a circle). For these transformations, we modify the formula
\Cref{eq:general_formulation} by differentiating in a parameterization of $\SE(2)$: concretely, we
use
\begin{equation*}
  \SO(2) \cong \set*{
    \begin{bmatrix}
      \cos \theta & -\sin \theta \\
      \sin \theta & \cos \theta
  \end{bmatrix} \given \theta \in [0, 2\pi]}.
\end{equation*}
Writing $F : \bbR \to \bbR^{m \times n \times 2}$ for this parameterization composed with our
usual vector field representation \Cref{eq:affine_subspace_centered} for subgroups of the affine
transforms, we modify the objective \Cref{eq:general_formulation} to be $\min_{\theta} \varphi(\vy
\circ F(\theta))$. A simple calculation then shows that gradients in this parameterization are
obtainable from gradients with respect to the affine parameterization as
\begin{equation*}
  \nabla_{\theta}[\varphi(\vy \circ F)](\theta)
  =
  \ip*{
    \nabla_{\vA}[\varphi(\vy \circ \vtau_{\spcdot, \vb})]\left(
      \begin{bmatrix}
        \cos \theta & -\sin \theta \\
        \sin \theta & \cos \theta
      \end{bmatrix}
    \right)
    }{
    \begin{bmatrix}
      -\sin \theta & -\cos \theta \\
      \cos \theta & -\sin \theta
    \end{bmatrix}
  }.
\end{equation*}
This is a minor extra nonlinearity that replaces the proximal operation when we unroll networks
as in \Cref{eq:prox_general} with this motion model.  Gradients and projections with respect to
the translation parameters are no different from the affine case.

\subsection{Gradient Calculations for Unrolled Network Architectures} \label{sec:grads}
We collect in this section several computations relevant to gradients of the function $\varphi$
(following the structure of \Cref{eq:general_formulation}) in the optimization formulations
\Cref{eq:formulation_reg,eq:formulation_costsmooth,eq:formulation_withbg,eq:complementary_smoothing}.

\paragraph{$\vtau$ gradients.} All of the costs we consider use the $\ell^2$ error $\norm{}_F$, so
their gradient calculations with respect to $\vtau$ are very similar. We will demonstrate the
gradient calculation for \Cref{eq:formulation_reg} to show how \Cref{eq:registration_unrolled} is
derived; the calculations for other costs follow the same type of argument. To be concise, we will
write $\nabla_{\vtau}\varphi$ for the gradient with respect to $\vtau$ of the relevant costs
$\varphi(\vy \circ \vtau)$.

\begin{proposition}
  \label{prop:reg_gradient}
  Let $\varphi$ denote the $\norm{}_F$ cost in \Cref{eq:formulation_reg}. One has
  \begin{equation*}
    \nabla_{\mb \tau} \varphi(\vtau) 
    =
    \sum_{k=0}^{c-1} 
    \left(\vg_{\sigma^2} \conv \sP_{\Omega}\left[
        \vg_{\sigma^2} \conv \left( \vy \circ \vtau - \vx_o \right)_k
    \right] \kron \One_2 \right) 
    \circleddot \left(\diff \vy_k \circ \vtau\right).
  \end{equation*}
  \begin{proof}
    The cost separates over channels, so by linearity of the gradient it suffices to assume $c=1$.
    We proceed by calculating the differential of $\varphi(\vy \circ \vtau)$ with respect to
    $\vtau$. We have for $\vDelta$ of the same shape as $\vtau$ and $t \in \bbR$
    \begin{equation*}
      \dac \varphi(\vy \circ (\vtau + t\vDelta))
      =
      \ip*{
        \sP_{\Omega}\left[
          \vg_{\sigma^2} \conv (\vy \circ \vtau - \vx_o)
        \right] \kron \One_2
        }{
        \sP_{\Omega}\left[
          \vg_{\sigma^2} \conv \left((\diff \vy \circ \vtau)
          \circleddot \vDelta\right)
        \right]
      },
    \end{equation*}
    where $\diff \vy \in \bbR^{m \times n \times 2}$ is the Jacobian matrix of $\vy$, defined as
    (here $\dot{\phi}$ denotes the derivative of the cubic convolution interpolation kernel
    $\phi$)
    \begin{equation*}
      \diff \vy_0 = 
      \sum_{(k, l) \in \bbZ^2} y_{kl} \dot{\phi}(\vm \One\adj - k\One) \circleddot \phi(\One
      \vn\adj - l\One),
      \quad
      \diff \vy_1 = 
      \sum_{(k, l) \in \bbZ^2} y_{kl} {\phi}(\vm \One\adj - k\One) \circleddot \dot{\phi}(\One
      \vn\adj - l\One),
    \end{equation*}
    and where for concision we are writing $\vg_{\sigma^2} \conv \diff \vy$ to denote the filtering
    of each of the two individual channels of $\diff \vy$ by $\vg_{\sigma^2}$.
    Using three adjoint relations ($\sP_{\Omega}$ is an orthogonal projection, hence self-adjoint;
    the adjoint of convolution by $\vg_{\sigma^2}$ is cross-correlation with $\vg_{\sigma^2}$;
    elementwise multiplication is self-adjoint) and a property of the tensor product, the claim
    follows.
  \end{proof}
\end{proposition}

\paragraph{Convolutional representation of cost-smoothed formulation
\Cref{eq:formulation_costsmooth}.} The cost-smoothed formulation \Cref{eq:formulation_costsmooth} 
can be directly expressed as a certain convolution with $\vg_{\sigma^2}$, leading to very fast
convolution-free inner loops in gradient descent implementation. To see this, write
\begin{align*}
  \norm*{
    \sP_{\Omega}\left[
      \vy \circ (\vtau + \vtau_{\Zero, \vDelta}) - \vx_o
    \right]
  }_F^2
  &=
  \norm*{ \sP_{\Omega}\left[ \vy \circ (\vtau + \vtau_{\Zero, \vDelta}) \right] }_F^2
  +
  \norm*{ \sP_{\Omega}\left[ \vx_o \right] }_F^2
  +
  2 \ip*{ \sP_{\Omega}\left[ \vy \circ (\vtau + \vtau_{\Zero, \vDelta})\right]}
  {\sP_{\Omega}\left[ \vx_o \right]} \\
  &=
  \ip*{\left[ \vy \circ (\vtau + \vtau_{\Zero, \vDelta})\right]^{\circleddot
  2}}{\sP_{\Omega}\left[\One\right]}
  +
  \norm*{ \sP_{\Omega}\left[ \vx_o \right] }_F^2
  +
  2 \ip*{ \vy \circ (\vtau + \vtau_{\Zero, \vDelta})}
  {\sP_{\Omega}\left[ \vx_o \right]},
\end{align*}
using self-adjointness of $\sP_{\Omega}$ and the fact that it can be represented as an elementwise
multiplication, and writing $[\spcdot]^{\circleddot 2}$ for elementwise squaring.
Thus, denoting the $\norm{}_F$ cost in \Cref{eq:formulation_costsmooth} by $\varphi(\vtau)$,
$\varphi$ can be written as
\begin{align*}
  2\varphi(\vtau)
  =
  &\ip*{ \sum_{\vDelta } (\vg_{\sigma^2})_{\vDelta} \left[ \vy
  \circ (\vtau + \vtau_{\Zero, \vDelta})\right]^{\circleddot 2}}{\sP_{\Omega}\left[\One\right]}
  + \ip*{\vg_{\sigma^2}}{\One} \norm*{ \sP_{\Omega}\left[ \vx_o \right] }_F^2\\
  &\quad   +
  2 \ip*{ \sum_{\vDelta} (\vg_{\sigma^2})_{\vDelta} \cdot \vy
  \circ (\vtau + \vtau_{\Zero, \vDelta})} {\sP_{\Omega}\left[ \vx_o \right]}.
\end{align*}
This can be expressed as a cross-correlation with $\vg_{\sigma^2}$:
\begin{equation*}
  2\varphi(\vtau)
  =
  \ip*{ \vg_{\sigma^2} \conv \left[\vy \circ \vtau\right]^{\circleddot
  2}}{\sP_{\Omega}\left[\One\right]}
  + \ip*{\vg_{\sigma^2}}{\One} \norm*{ \sP_{\Omega}\left[ \vx_o \right] }_F^2
  +
  2 \ip*{ \vg_{\sigma^2} \conv (\vy\circ \vtau) } {\sP_{\Omega}\left[ \vx_o \right]},
\end{equation*}
and taking adjoints gives finally
\begin{equation*}
  2\varphi(\vtau)
  =
  \ip*{ \left[\vy \circ \vtau\right]^{\circleddot 2}}{\vg_{\sigma^2} \conv
  \sP_{\Omega}\left[\One\right]}
  + \ip*{\vg_{\sigma^2}}{\One} \norm*{ \sP_{\Omega}\left[ \vx_o \right] }_F^2 + 2 \ip*{
  \vy\circ \vtau} {\vg_{\sigma^2} \conv \sP_{\Omega}\left[ \vx_o \right]}.
\end{equation*}
This gives a convolution-free gradient step implementation for this cost (aside from pre-computing
the fixed convolutions in the cost), and also yields a useful interpretation of the cost-smoothed
formulation \Cref{eq:formulation_costsmooth}, and its disadvantages relative to the
background-modeled formulation \Cref{eq:formulation_withbg}.

\newcommand{\nadj}{^{-*}}
\paragraph{Filter gradient for complementary smoothing formulation
\Cref{eq:complementary_smoothing}.}
Relative to the standard registration model formulation \Cref{eq:formulation_reg}, the
complementary smoothing spike registration formulation \Cref{eq:complementary_smoothing} contains
an extra complicated transformation-dependent gaussian filter. We provide a key lemma below for the
calculation of the gradient with respect to the parameters of the complementary smoothing cost in
``standard parameterization'' (see the next paragraph below). The full calculation follows
the proof of \Cref{prop:reg_gradient} with an extra ``product rule'' step and extra adjoint
calculations.
\begin{proposition}
  \label{prop:complementary_smoothing_filter_grad}
  Given fixed $\sigma^2 > \sigma_0^2 > 0$, define $\vSigma(\vA) = \sigma^2 \vI - \sigma_0^2
  (\vA\adj \vA)\inv$, and define
  \begin{equation*}
    \vg(\vA) = \sqrt{\det(\vA\adj \vA)}\vg_{\vSigma(\vA)},
  \end{equation*}
  where the filter is $m$ by $n$ and the domain is the open set $\set{\vA \given \vsigma \vI -
  \vsigma_0^2(\vA\adj\vA)\inv \succ \Zero}$. Then for any fixed $\vV \in \bbR^{m \times n}$, one
  has
  \begin{equation*}
    \nabla_{\vA}[ \ip{\vV}{\vg} ](\vA) =
    \sigma_0^2\vA\nadj\left(
      \vSigma(\vA)\inv\left( \sum_{i,j} V_{ij} g(\vA)_{ij}  \vw_{ij} \vw_{ij}\adj\right)
      \vSigma(\vA)\inv - \ip{\vg(\vA)}{\vV} \vSigma(\vA)\inv
    \right)(\vA\adj \vA)\inv + \ip{\vg(\vA)}{\vV}\vA\nadj,
  \end{equation*}
  where $\vA\nadj = (\vA\inv)\adj$.
  \begin{proof}
    For $(i, j) \in \set{0, \dots, m-1} \times \set{0, \dots, n-1}$, let $\vw_{ij} = (i -
    \floor{m/2}, j - \floor{n/2})$. Then we have
    \begin{equation*}
      \vg(\vA) = \frac{1}{2\pi} \sum_{i,j} \ve_{ij} \exp \left(
        -\frac{1}{2} \vw_{ij}\adj \vSigma(\vA)\inv \vw_{ij}
        -\frac{1}{2} \log\det \Sigma(\vA) + \frac{1}{2} \log\det \vA\adj \vA
      \right).
    \end{equation*}
    Let $\diff \vg$ denote the differential of $\vA \mapsto \vg(\vA)$. By the chain rule, we have
    for any $\vDelta \in \bbR^{2\times 2}$
    \begin{equation*}
      \ip*{\vV}{\diff \vg_{\vA}(\vDelta)}
      = 
      \frac{1}{2}
      \sum_{i,j} V_{ij} g(\vA)_{ij} 
      \dac \left[
        -\vw_{ij}\adj \vSigma(\vA+t\vDelta)\inv \vw_{ij} -\log\det \Sigma(\vA+t\vDelta) 
        + \log\det (\vA+t\vDelta)\adj (\vA+t\vDelta)
      \right].
    \end{equation*}
    We need the differential of several mappings here. We will use repeatedly that if $\vX \in
    \GL(2)$ and $\vW\in \bbR^{2\times 2}$, one has
    \begin{equation}
      \diff[ \vX \mapsto \ip*{\vW}{\vX\inv}]_{\vX}(\vDelta) = - \ip*{\vDelta}{\vX\nadj \vW
      \vX\nadj}.
      \label{eq:matrix_inv_differential}
    \end{equation}
    Applying \Cref{eq:matrix_inv_differential} and the chain rule, we get
    \begin{align*}
      \diff[ \ip*{\vW}{\vSigma}]_{\vA}(\vDelta)
      &= \sigma_0^2 \ip*{(\vA\adj \vA)\inv \vW (\vA\adj \vA)\inv}{\vDelta\adj
      \vA + \vA\adj \vDelta}\\
      &= \sigma_0^2 \ip*{\vA\nadj (\vW + \vW\adj) (\vA\adj \vA)\inv}{\vDelta}.
      \labelthis \label{eq:cov_differential}
    \end{align*}
    In particular, using the chain rule and \Cref{eq:matrix_inv_differential,eq:cov_differential}
    gives
    \begin{equation}
      \dac \left[
        \vw_{ij}\adj \vSigma(\vA + t\vDelta)\inv \vw_{ij}
      \right]
      =
      -2\sigma_0^2\ip*{\vA\nadj \vSigma(\vA)\inv \vw_{ij} \vw_{ij}\adj \vSigma(\vA)\inv (\vA\adj
      \vA)\inv}{\vDelta}.
      \label{eq:differential_term1}
    \end{equation}
    Next, using the Leibniz formula for the determinant, we obtain
    \begin{equation}
      \diff [\log\det ]_{\vX}(\vDelta) = \ip*{\vX\nadj}{\vDelta}.
      \label{eq:logdet_differential}
    \end{equation}
    The chain rule and \Cref{eq:cov_differential,eq:logdet_differential} thus give
    \begin{equation}
      \dac \left[
        \log\det \Sigma(\vA+t\vDelta) 
      \right]
      = 2\sigma_0^2\ip*{\vA\nadj \vSigma(\vA)\inv (\vA\adj \vA)\inv}{\vDelta},
      \label{eq:differential_term2}
    \end{equation}
    and similarly
    \begin{equation}
      \dac \left[
        \log\det (\vA+t\vDelta)\adj (\vA+t\vDelta)
      \right]
      = 2\ip*{\vA\nadj}{\vDelta}.
      \label{eq:differential_term3}
    \end{equation}
    Combining \Cref{eq:differential_term1,eq:differential_term2,eq:differential_term3}, we have
    \begin{equation*}
      \ip*{\vV}{\diff \vg_{\vA}(\vDelta)}
      = 
      \sum_{i,j} V_{ij} g(\vA)_{ij} 
      \ip*{
        \sigma_0^2\vA\nadj\left(
          \vSigma(\vA)\inv \vw_{ij} \vw_{ij}\adj \vSigma(\vA)\inv - \vSigma(\vA)\inv
        \right)(\vA\adj \vA)\inv
        + \vA\nadj
      }{\vDelta},
    \end{equation*}
    and the claim follows by distributing and reading off the gradient.\footnote{After
      distributing, the sum over $i,j$ in the first factor can be computed relatively efficiently
    using a Kronecker product.}
  \end{proof}
\end{proposition}

\paragraph{Differentiating costs in ``inverse parameterization''.} Our theoretical study of spike
alignment in \Cref{app:mc-spike-affine} and our experiments on the discretized objective
\Cref{eq:complementary_smoothing} in \Cref{sec:mcs_experiments} suggest strongly to prefer ``inverse
parameterization'' relative to standard parameterization of affine transformations for
optimization. By this, we mean the following: given a cost $\varphi(\vtau_{\vA, \vb})$ optimized
over affine transformations $(\vA, \vb)$, one optimizes instead $\varphi(\vtau_{\vA\inv,
-\vA\inv\vb})$. This nomenclature is motivated by, in the continuum, the inverse of the affine
transformation $\vx \mapsto \vA \vx + \vb$ being $\vx \mapsto \vA\inv(\vx - \vb)$. Below, we show
the chain rule calculation that allows one to easily obtain gradients for inverse-parameterized
objectives as linear corrections of the standard-parameterized gradients.

\newcommand{\VA}{\vDelta_{\vA}}
\newcommand{\VB}{\vDelta_{\vb}}
\begin{proposition}
  \label{prop:inverse_parameterization}
  Let $\varphi : \bbR^{2\times 2} \times \bbR^2 \to \bbR$, and let $F(\vA, \vb) = (\vA\inv,
  -\vA\inv \vb)$ denote the inverse parameterization mapping, defined on $\GL(2) \times \bbR^2$.
  One has
  \begin{align*}
    \nabla_{\vA}[\varphi\circ F](\vA, \vb) &= -\vA\nadj\left( \nabla_{\vA}[\varphi]\circ
    F(\vA,\vb)\right) \vA\nadj + \vA\nadj\left(\nabla_{\vb}[\varphi] \circ F(\vA, \vb)\right) (\vA
    \inv
    \vb)\adj, \\
    \nabla_{\vb}[\varphi\circ F](\vA, \vb) &= -\vA\nadj\left(\nabla_{\vb}[\varphi] \circ F(\vA,
    \vb)\right),
  \end{align*}
  where $\vA\nadj = (\vA\inv)\adj$.
  \begin{proof}
    Let $\diff[\varphi \circ F]$ denote the differential of $\varphi \circ F$ (and so on). We have
    for $\vDelta_{\vA}$ and $\vDelta_{\vb}$ the same shape as $\vA$ and $\vb$
    \begin{align*}
      \diff F_{\vA, \vb}(\VA, \VB) &= \dac\left(
        (\vA + t \VA)\inv, -(\vA + t \VA)\inv(\vb + t\VB)
      \right) \\
      &= \left( -\vA \inv \VA \vA\inv, -\left( \vA\inv \VB - \vA\inv \VA \vA\inv \vb
      \right)\right)
    \end{align*}
    where the asserted expression for the derivative through the matrix inverse follows from, say,
    the Neumann series. Now, the chain rule and the definition of the gradient imply
    \begin{equation*}
      \diff[ \varphi \circ F]_{\vA, \vb}(\VA, \VB)
      = \ip*{
        \left(
          \nabla_{\vA}[\varphi \circ F](\vA, \vb),
          \nabla_{\vb}[\varphi \circ F](\vA, \vb)
        \right)
        }{
        \left( -\vA \inv \VA \vA\inv, -\left( \vA\inv \VB - \vA\inv \VA \vA\inv \vb
        \right)\right)
      },
    \end{equation*}
    and the claim follows by distributing and taking adjoints in order to read off the gradients
    from the previous expression.
  \end{proof}
\end{proposition}

We remark that centering, as discussed in \Cref{sec:parametric}, can be implemented identically
to the standard parameterization case when using inverse parameterization.

\subsection{Additional Experiments and Experimental Details} \label{sec:extra_experiment_details}

\paragraph{General details for experiments.} We use normalized cross correlation (NCC) and
zero-normalized cross correlation (ZNCC) for measuring the performance of registration on textured
and spike data respectively. Specifically, for two multichannel images $\mb X, \mb Y \in
\bb{R}^{m\times n\times c}$, let $\tilde{\mb X}$ and $\tilde{\mb Y}$ be the channel-wise
mean-subtracted images from $\mb X$ and $\mb Y$. The quantities NCC and ZNCC are defined as
$\mathrm{NCC}(\vX, \vY) = \frac{\ip{\mb X}{\mb Y}}{\|{\mb X}\|_F \|{\mb Y}\|_F}$ and
$\mathrm{ZNCC}(\vX, \vY) = \frac{\ip{\tilde{\mb X}} {\tilde{\mb Y}}}{\|\tilde{\mb X}\|_F
\|\tilde{\mb Y}\|_F}$.

\subsubsection{\Cref{fig:centerpiece} Experimental Details}

In the experiment comparing the complexity of optimization and covering-based methods for
textured motif detection shown in \Cref{fig:centerpiece}, the raw background image used has
dimension $2048\times 1536$. The crab template is first placed at the center, then a random
transformation is applied to generate the scene $\vy$. Translation consists of random amounts on
both $x$ and $y$ directions uniformly in $[-5, 5]$ pixels. Euclidean transforms in addition apply
a rotation with angle uniformly from $[-\frac{\pi}{4}, \frac{\pi}{4}]$. Similarity transforms in
addition applies a scaling uniformly from $[0.8, 1.25]$. Generic affine transforms are
parameterized by a transformation matrix $\mb A\in\mathbb{R}^{2\times 2}$ and offset vector $\mb
b\in \mathbb{R}^2$, with the singular values of $\mb A$ uniformly from $[0.8, 1.25]$ and the left
and right orthogonal matrices being rotation matrices with angle uniformly in $[-\frac{\pi}{4},
\frac{\pi}{4}]$. For each of the 4 modes of transform, 10 random images are generated. 
The optimization formulation used is \Cref{eq:formulation_withbg}, with $\vx_o$ the crab body
motif shown in \Cref{fig:motif}(a).
The
optimization-based method uses a multi-scale scheme, which uses a sequence of decreasing values of
$\sigma$ and step sizes, starting at $\sigma=5$ and step size $0.005\sigma$ (except for affine
mode which starts at $\sigma=10$), with $\sigma$ halved every 50 iterations until stopping
criteria over the ZNCC is met, where ZNCC is calculated over the motif support $\Omega$ only. For
each value of $\sigma$, a dilated support $\tilde{\Omega}$ is used, which is the dilation of
$\Omega$ two $\sigma$ away from the
support of the motif. The background model covers the region up to $5\sigma$ away from the motif.
The background $\bm\beta$ is initialized as a gaussian-smoothed version of the difference between
the initialized image and the ground truth motif, and then continuously updated in the
optimization. For the first 5 iterations of every new scale $\sigma$, only the background is
updated while the transformation parameters are held constant. The covering-based method samples a
random transform from the corresponding set of transforms used in each try.

\subsubsection{\Cref{fig:convergence} Experimental Details}

In the experiment of verifying the convergence of multichannel spike registration as shown in
\Cref{fig:convergence}, the motif consists of 5 spikes placed at uniformly random positions
in a $61\times 81$ image. To allow the spike locations to take non-integer values, we represent
each spike as a gaussian density with standard deviation $\sigma_0=3$ centered at the spike
location, and evaluated on the grid. A random affine transformation of the motif is generated as
the scene. As a result, we are able to use this $\sigma_0$-smoothed input in
\Cref{eq:complementary_smoothing} \textit{without extra smoothing}, and we can compensate the
variance of the filter applied to $\vx_o$ in the formulation to account for the fact that we
already smoothed by $\sigma_0$ when generating the data.
The smoothing level in the registration is chosen according to equation
\Cref{eq:sigma-bound} in \Cref{thm:multichannel-spike}. Due to the discretization effect and
various artifacts, the step sizes prescribed in \Cref{thm:multichannel-spike} will lead to
divergence, so we reduce the step sizes by multiplying a factor of 0.2.

\subsubsection{Further Experimental Details}

The beach background used for embedding the crab template throughout the experiments is
\texttt{CC0}-licensed and available online:
\url{https://www.flickr.com/photos/scotnelson/28315592012}.
Our code and data are available at \codeurl.

\subsection{Canonized Object Preprocessing and Calibration for Hierarchical Detection}
\label{sec:hier_preprocessing}

The hierarchical detection network implementation prescription in \Cref{sec:hier_networks}
assumes the occurrence maps $\vx_v$ for $v \in V \setminus \set{1, \dots, K}$ are given; in
practice, these are first calculated using the template $\vy_o$ and its motifs, by a process we
refer to as extraction. Simultaneously, to extract these occurrence maps and have them be useful
for subsequent detections it is necessary to have appropriate choices for the various
hyperparameters involved in the network: we classify these as `registration' hyperparameters (for
each $v \in V$, the step size $\nu_v$; the image, scene, and input smoothing parameters
$\sigma_v^2$, $\sigma_{0, v}^2$, and $\sigma_{\mathrm{in}}^2$; the number of registration
iterations $T_v$; and the vertical (``height'') and horizontal (``width'') stride sizes
$\Delta_{H, v}$ and $\Delta_{W, v}$) or `detection' hyperparameters (for each $v \in V$, the
suppression parameter $\alpha_v$ and the threshold parameter $\gamma_v$). We describe these issues
below, as well as other relevant implementation issues.

\paragraph{Hyperparameter selection.} We discuss this point first, because it is necessary to
process the `leaf' motifs before any occurrence maps can be extracted. In practice, we `calibrate'
these hyperparameters by testing whether detections succeed or fail given the canonized template
$\vy_o$ as input to the (partial) network. Below, we first discuss hyperparameters related to
visual motifs (i.e., the formulation \Cref{eq:hier_strided_textured}), then hyperparameters for
spiky motifs (i.e., the formulation \Cref{eq:hier_strided_spiky}).

\subparagraph{Stride density and convergence speed:} The choice of these parameters encompasses
a basic computational tradeoff: setting $T_v$ larger allows to leverage the entire basin of
attraction of the formulations \Cref{eq:hier_strided_textured,eq:hier_strided_spiky}, enabling
more reliable values of $\min_{\vlambda \in \Lambda_v}\mathrm{loss}(v, \vlambda)$ and the
use of larger values of $\Delta_{H, v}$ and $\Delta_{W, v}$; however, it requires more
numerical operations (convolutions and interpolations) for each stride $\vlambda \in
\Lambda_{v}$. 
In our experiments we err on the side of setting $T_v$ large, and tune the stride sizes
$\Delta_{W, v}$ and $\Delta_{H, v}$ over multiples of $4$ (setting them as large as possible
while being able to successfully detect motifs).
The choice of the step sizes $\nu_v$ is additionally complicated by the smoothing and
motif-dependence of this parameter. As we describe in \Cref{sec:unrolling}, we 
treat the step sizes taken on each component of $(\vA, \vb)$ independently, and in our
experiments use a small multiple (i.e., $1/10$) of $t_v^{\vA} = 4\sigma/\max\set{m_v^2,
n_v^2}$ and $t_v^{\vb} = 2\sigma /\max \set{m_v, n_v}$ for all
visual motifs. %
This prescription is a heuristic that we find works well for the motifs and
smoothing parameters (see the next point below) we test, inspired by the theoretical
prescriptions in \Cref{sec:theory} for spike alignment that we discuss later in this section.

\subparagraph{Smoothing parameters:} The smoothing level $\sigma_v^2$ in
\Cref{eq:hier_strided_textured} increases the size of the basin of attraction when set larger.
For this specific formulation, we find it more efficient to expand the basin by striding, and 
enforce a relatively small value of $\sigma^2_v = 9$ for all visual motifs. 
Without input
smoothing, we empirically observe that the first-round-multiscale cost-smoothed formulation
\Cref{eq:hier_strided_textured} is slightly unstable with respect to high-frequency content in
$\vy$: this motivates us to introduce this extra smoothing with variance
$\sigma_{\mathrm{in}}^2 = 9/4$, which removes interpolation artifacts that hinder convergence.
We find the multiscale smoothing mode of operation described in \Cref{sec:hier_networks} to be
essential for distinguishing between strides $\vlambda$ which have ``failed'' to register the
motif $\vx_v$ and those that have succeeded, through the error $\mathrm{loss}(\vlambda, v)$:
in all experiments, we run the second-phase multiscale round for
\Cref{eq:hier_strided_textured} as described in \Cref{sec:hier_networks}, for $256$ iterations
and with $\sigma^2 = 10^{-2}$ and $\sigma_{\mathrm{in}}^2 = 10^{-12}$.
We describe the choice of $\sigma_{0, v}^2$ below, as it is more of a spike registration
hyperparameter (c.f.\ \Cref{eq:hier_strided_occmaps}).

\subparagraph{Detection parameters:} The scale parameters $\alpha_v$ are set based on the size
of the basin of attraction around the true transformation of $\vx_v$, and in particular on the
scale of $\mathrm{loss}(\vlambda, v)$ at ``successes'' and ''failures'' to register. In our
experiments, we simply set $\alpha_v = 1$ for visual motifs. The choice of the threshold
parameter $\gamma_v$ is significantly more important: it accounts for the fact that the final
cost $\mathrm{loss}(\vlambda, v)$ at a successful registration is sensitive to both the motif
$\vx_v$ \textit{and} the background/visual clutter present in the input $\vy$. In our
experiments in \Cref{sec:hier_results}, we tune the parameters $\gamma_v$ on a per-motif basis
by calculating $\mathrm{loss}(\vlambda, v)$ for embeddings $\vy_o \circ \vtau_0$ for $\vtau_0
\in \SO(2)$ up to some maximum rotation angle in visual clutter, classifying each $\vlambda$
as either a successful registration or a failure, and then picking $\vgamma_v$ to separate the
successful runs for all rotation angles from the failing runs. For the motifs and range of
rotation angles we consider, we find that such a threshold always exists. However, at larger
rotation angles we run into issues with the left and right eye motifs being too similar to
each other, leading to spurious registrations and the non-existence of a separating threshold.
In practice, this calibration scheme also requires a method of generating visual clutter that
matches the environments one intends to evaluate in. The calibrated threshold parameters used
for our experiments in \Cref{sec:hier_results} are available in our released implementation.

\subparagraph{Hyperparameters for spiky motifs:} The same considerations apply to hyperparameter
selection for spiky motifs (i.e., the formulation \Cref{eq:hier_strided_spiky}). However, the
extra structure in such data facilitates a theoretical analysis that corroborates the intuitive
justifications for hyperparameter tradeoffs we give above and leads to specific prescriptions for
most non-detection hyperparameters, allowing them to be set in a completely tuning-free fashion.
We present these results in \Cref{sec:theory}. For detection hyperparameters, we follow the same
iterated calibration process as for visual motifs, with scale parameters $\alpha_v = 2.5 \cdot
10^5$ (typical values of the cost \Cref{eq:hier_strided_spiky} are much smaller than those of the
cost \Cref{eq:hier_strided_textured}, due to the fact that the gaussian density has a small $L^2$
norm).  For the occurrence map smoothing parameters $\sigma_{0, v}^2$, our network construction
above necessitates setting these parameters to be the same for all $v \in V$; we find empirically
that a setting $\sigma_{0, v}^2 = 9$ is sufficient to avoid interpolation artifacts. 
Finally, the bounding box masks $\Omega_{v}$ are set during the extraction process (see below),
and are dilated by twice the total size of the filters $\vg_{\sigma_v^2}$. In practice, when
implementing gaussian filters, we make the image size square, with side lengths $6\sigma$
(rounded to the next largest odd integer).

\paragraph{Occurrence map extraction.} 

Although the criteria above (together with the theoretical guidance from \Cref{sec:theory}) are
sufficient to develop a completely automatic calibration process for the various hyperparameters
above, in practice we perform calibration and occurrence map extraction in a `human-in-the-loop'
fashion. The extraction process can be summarized as follows (it is almost identical to the
detection process described in \Cref{sec:hier_networks}, with a few extra steps implicitly
interspersed with calibration of the various hyperparameters):
\begin{enumerate}
  \item \textbf{Use the canonized template as input:} We set $\vy_o$ as the network's input.
  \item \textbf{Process leaf motifs:} Given suitable calibrated settings of the hyperparameters
    for leaf motifs $v \in V$, perform detection and generate all occurrence maps $\vomega_v$ via
    \Cref{eq:hier_strided_occmaps}.
  \item \textbf{Extract occurrence motifs at depth $\diam(G) - 1$:} For each $v$ with $d(v) =
    \diam(G) - 1$, we follow the assumptions made in \Cref{sec:hier_data_model} (in particular,
    that each visual motif occurs only once in $\vy_o$ and $G$ is a tree) and after aggregating
    the occurrence map from $v$'s child nodes via \Cref{eq:multichannel_spike_map}, we extract
    $\vx_v$ from $\vy_v$ by cropping to the bounding box for the support of $\vy_v$. Technically,
    since \Cref{eq:hier_strided_occmaps} uses a gaussian filter, the support will be nonzero
    everywhere, and instead we threshold at a small nonzero value (e.g.\ $1/20$ in our
    experiments) to determine the ``support''.
  \item \textbf{Continue to the root of $G$:} Perform registration to generate the occurrence maps
    for nodes at depth $\diam(G) - 1$, then continue to iterate the above steps until the root
    node is reached and processed.
\end{enumerate}
Note that the extracted occurrence motifs $\vx_v$ for $v \in V \setminus \set{1, \dots, K}$ depend
on proper settings of the registration and detection hyperparameters: if these parameters are set
imprecisely, the extracted occurrence maps will not represent ideal detections (e.g. they may not
be close to a full-amplitude gaussian at the locations of the motifs in $\vy_o$ as they should, or
they may not suppress failed detections enough).

\paragraph{Other implementation issues.}
The implementation issue of centering, discussed in \Cref{sec:parametric}, is relevant to the
implementation of the unrolled solvers for \Cref{eq:hier_strided_spiky,eq:hier_strided_occmaps}.
We find that a useful heuristic is to center the transformation $\vtau$ at the location of the
center pixel of the embedded motif $\vx_v$ (i.e., for a stride $\vlambda \in \Lambda_v$, at the
coordinates $\vlambda + ( (m_v - 1)/2, (n_v - 1)/2)$). To implement this centering, the locations
of the detections in the spike map definition \Cref{eq:hier_strided_occmaps} need to have the
offsets $( (m_v - 1)/2, (n_v - 1)/2)$ added.

The network construction in \Cref{sec:hier_networks} relies on the extraction process described
above to employ an identical enumeration strategy in the traversal of the graph $G$ as the
detection process (i.e., assuming that nodes are ordered in increasing order above). In our
implementation described in \Cref{sec:hier_results}, we instead label nodes arbitrarily when
preparing the network's input, and leave consistent enumeration of nodes during traversal to the
NetworkX graph processing library \cite{SciPyProceedings_11}.

\section{Proof of \Cref{thm:multichannel-spike}} \label{app:mc-spike-affine}
We consider a continuum model for multichannel spike alignment, motivated by the higher-level
features arising in the hierarchical detection network developed in \Cref{sec:motif}: signals
$\vX$ are represented as elements of $\bbR^{\bbR \times \bbR \times C}$, and are identifiable with
$C$-element real-valued vector fields on the (continuous, infinite) image plane $\bbR^2$. In this
setting, we write $\norm{\vX}_{L^2}^2 = \sum_{i=1}^C \norm{\vX_i}_{L^2}^2$ for the natural product
norm (in words, the $\ell^2$ norm of the vector of channelwise $L^2$ norms of $\vX$). For $\vp \in
\bbR^2$, let $\vdelta_{\vp} \in \bbR^{\bbR \times \bbR}$ denote a Dirac distribution centered at
$\vp$, defined via
\begin{equation*}
  \int_{\bbR^2} \vdelta_{\vp}(\vx) f(\vx) \diff \vx = f(\vp)
\end{equation*}
for all Schwartz functions $f$ \cite[\S I.3]{Stein1971-ed}. This models a `perfect' spike signal.
For $\vp \in \bbR^2$ and $\vM \in \bbR^{2\times 2}$ positive semidefinite, let $\vg_{\vp, \vM}$
denote the gaussian density on $\bbR^2$ with mean $\vp$ and covariance matrix $\vM$.
Consider a target signal  
\begin{equation} \label{eqn:X0-MCSM}
    \mb X_o = \sum_{i = 1}^c \mb \delta_{\mb v_i} \otimes  \mb e_i, 
\end{equation}
and an observation  
\begin{equation} \label{eqn:X-MCSM}
    \mb X = \sum_{i=1}^c  \mb \delta_{\mb u_i} \otimes \mb e_i
\end{equation}
satisfying 
\begin{equation} \label{eqn:MCSM-spike-rel}
    \mb v_i = \mb A_\star \mb u_i + \mb b_\star
\end{equation}
for some $(\vA_{\star}, \vb_{\star}) \in \GL(2) \times \bbR^2$. These represent the unknown
ground-truth affine transform to be recovered.
Consider the objective function
\begin{equation*}
  \varphi_{L^2,\sigma}(\mb A, \mb b) \equiv \frac{1}{2c} \left\| \vg_{\mb 0,\sigma^2
    \mb I - \sigma_0^2 (\mb A^* \mb A)^{-1}} \conv \left( \mr{det}^{1/2}( \mb A^* \mb A ) \left(
  \vg_{\mb 0,\sigma_0^2 \mb I} \conv \mb X \right) \circ \mb \tau_{\mb A, \mb b} \right) -
  \vg_{\mb 0,\sigma^2 \mb I} \conv {\mb X_o} \right\|_{L^2}^2, 
\end{equation*}
where $\vA\adj$ denotes the transpose, convolutions are applied channelwise, and for a signal $\vS
\in \bbR^{\bbR \times \bbR \times c}$, $\vS \circ \vtau_{\vA, \vb}(u, v) = \vS(a_{11} u + a_{12} v
+ b_{1}, a_{21}u + a_{22}v + b_2)$.
We study the following ``inverse parameterization'' of this function: 
\begin{equation*}
 \varphi^{\mr{inv}}_{L^2,\sigma}(\mb A, \mb b) \equiv \varphi_{L^2,\sigma}( \mb A^{-1}, - \mb A^{-1} \mb b ).
\end{equation*}
We analyze the performance of gradient descent for solving the optimization problem 
\begin{equation*}
        \min_{\mb A, \mb b} \varphi^{\mr{inv}}_{L^2,\sigma}(\mb A, \mb b).
\end{equation*}
Under mild conditions, local minimizers of this problem are global. Moreover, if $\sigma$ is set appropriately, the method exhibits linear convergence to the truth: 

\begin{theorem}[Multichannel Spike Model, Affine Transforms, $L^2$] \label{thm:spike_reg_app}
  Consider an instance of the
  multichannel spike model \eqref{eqn:X0-MCSM}-\eqref{eqn:X-MCSM}-\eqref{eqn:MCSM-spike-rel}, with
  $\mb U = [\mb u_1, \dots, \mb u_c ] \in \bb R^{2 \times c}$. Assume that the spikes $\mb U$ are
  centered and nondegenerate, so that $\mb U \One = \mb 0$ and $\rank(\vU) = 2$. Then
  gradient descent
  \begin{eqnarray*}
    \mb A_{k+1} &=& \mb A_k - \nu t_{\mb A} \nabla_{\mb A} \varphi^{\mr{inv}}_{L^2,\sigma}(\mb
    A_k, \mb b_k ), \\ 
    \mb b_{k+1} &=& \mb b_k - \nu t_{\mb b} \nabla_{\mb b} \varphi^{\mr{inv}}_{L^2,\sigma}(\mb
    A_k, \mb b_k ) 
  \end{eqnarray*}
  with smoothing 
  \begin{equation*}
    \sigma^2 
    \ge 2\frac{\max_i\, \norm{\vu_i}_2^2}{s_{\min}(\vU)^2} %
    \left(  
      {s_{\max}(\vU)^2 \| \mb A_\star - \mb I \|_F^2}%
      + c{  \| \mb b_\star \|_2^2 }%
    \right)
  \end{equation*}
  and step sizes 
  \begin{eqnarray*}
    t_{\mb A} &=& \frac{c}{ s_{\max}(\vU)^2 }, \\
    t_{\mb b} &=& 1, \\
    \nu &=& 8 \pi \sigma^4,
  \end{eqnarray*}
  from initialization $\mb A_0 = \mb I, \mb b_0 = \mb 0$ satisfies 
  \begin{eqnarray}
    t_{\mb A}^{-1} \| \mb A_k - \mb A_\star \|_{F}^2 + \| \mb b_k - \mb b_\star \|_2^2 &\le& \left( 1 - \frac{1}{2 \kappa} \right)^{2k} \Bigl(  t_{\mb A}^{-1} \| \mb I - \mb A_\star \|_{F}^2 + \| \mb b_\star \|_{2}^2 \Bigr), \label{eqn:opt-main-bound}
  \end{eqnarray}
  where 
  \begin{equation*}
    \kappa = \frac{ s_{\max}( \mb U)^2 }{ s_{\min}( \mb U)^2 },
  \end{equation*}
  with $s_{\min}(\vU)$ and $s_{\max}(\vU)$ denoting the minimum and maximum
  singular values of the matrix $\vU$. 
\end{theorem}

\begin{proof} 
  Below, we use the notation $\norm{\vM}_{\ell^p \to \ell^q} =
  \sup_{\norm{\vx}_p \leq 1} \norm{\vM \vx}_q$.
  We begin by rephrasing the objective function in a simpler form: by properties of
  the gaussian density,
\begin{eqnarray*}
  \varphi_{L^2,\sigma}({\mb A},{\mb b}) &=& \frac{1}{2c} \sum_{i=1}^c \left\|
  \vg_{{\mb A}^{-1} (\mb u_i -  {\mb b}),\sigma^2 \mb I} - \vg_{\mb v_i,\sigma^2 \mb I}
  \right\|_{L^2}^2,
\end{eqnarray*}
whence by an orthogonal change of coordinates
\begin{eqnarray*}
  \varphi_{L^2,\sigma}({\mb A},{\mb b}) &=& \frac{1}{c} \sum_{i=1}^c \psi\left( \tfrac{1}{2} \| {\mb A}^{-1}(\mb u_i - \mb b) - \mb v_i \|_2^2 \right),
\end{eqnarray*}
where 
\begin{eqnarray*}
\psi( t^2 / 2 ) &=& \tfrac{1}{2} \left\| \mb g_{t\mb e_1,\sigma^2\vI } - \mb g_{\mb 0,\sigma^2\vI} \right\|_{L^2}^2 \\
&=& \frac{1}{4\pi \sigma^2} - \innerprod{ \mb g_{t\mb e_1,\sigma^2\vI} }{ \mb g_{\mb 0,\sigma^2\vI} } \\
&=& \frac{1}{4\pi \sigma^2} - \frac{1}{(2 \pi\sigma^2)^2} \left(\int_{\bbR} e^{-s^2 / \sigma^2 } d
  s \right)\left( \int_{\bbR} e^{- (s - t)^2 / 2 \sigma^2 } e^{- s^2 / 2 \sigma^2 }
  ds\right) \\
  &=& \frac{1}{4\pi \sigma^2} - \frac{ 2^{-1/2}}{ (2 \pi \sigma^2)^{3/2} } \int_{\bbR}
  e^{-\frac{(s - t/2)^2}{\sigma^2}} e^{- \frac{t^2}{4 \sigma^2}} ds \\
  &=& \frac{1}{4\pi\sigma^2}\left(1 - \exp\left( - \frac{t^2/2}{2 \sigma^2} \right)\right).
\end{eqnarray*} 
So 
\begin{equation*}
    \varphi^{\mr{inv}}_{L^2,\sigma}(\mb A, \mb b) = \varphi_{L^2,\sigma}(\mb A^{-1},-\mb A^{-1}\mb b) = \frac{1}{c} \sum_{i=1}^c \psi( \tfrac{1}{2} \| \mb A \mb u_i + \mb b - \mb v_i \|_2^2 ). 
\end{equation*}
Differentiating, we obtain
\begin{eqnarray*}
    \nabla_{\mb A} \varphi^{\mr{inv}}_{L^2,\sigma}(\mb A,\mb b) &=& \frac{1}{c} \sum_{i=1}^c \dot{\psi}( \tfrac{1}{2} \| \mb \delta_i \|_2^2 ) \mb \delta_i \mb u_i^* \\
    \nabla_{\mb b} \varphi^{\mr{inv}}_{L^2,\sigma}(\mb A,\mb b) &=& \frac{1}{c} \sum_{i=1}^c \dot{\psi}( \tfrac{1}{2} \| \mb \delta_i \|_2^2 ) \mb \delta_i, 
\end{eqnarray*}
where for concision
\begin{align*}
  \mb \delta_i &= \mb A \mb u_i + \mb b - \mb v_i \\
  &= (\vA - \vA_{\star})\vu_i + \vb - \vb_{\star}.
\end{align*}
In these terms, we have the following expression for a single iteration of gradient descent:
\begin{eqnarray*}
\mb A^+ &=& \mb A - \frac{t_{\mb A}}{c} \sum_{i=1}^c \nu \dot{\psi}(\tfrac{1}{2} \| \mb \delta_i \|_2^2 ) \mb \delta_i \mb u_i^* \\
&=& \mb A - \frac{t_{\mb A}}{c} \sum_{i=1}^c \nu \dot{\psi}(\tfrac{1}{2} \| \mb \delta_i \|_2^2 ) (\mb A - \mb A_\star ) \mb u_i \mb u_i^* - \frac{t_{\mb A}}{c} \sum_{i=1}^c \nu \dot{\psi}(\tfrac{1}{2} \| \mb \delta_i \|_2^2 ) ( \mb b - \mb b_\star ) \mb u_i^* \\
&=& \mb A - \frac{\nu t_{\mb A}}{c} (\mb A - \mb A_\star) \mb U \dot{\mb \Psi} \mb U^* - \frac{\nu t_{\mb A}}{c} ( \mb b - \mb b_\star ) \dot{\mb \psi}^* \mb U^*
\end{eqnarray*}
and 
\begin{eqnarray*}
\mb b^+ &=& \mb b - \frac{t_{\mb b}}{c} \sum_{i=1}^c \nu \dot{\psi}(\tfrac{1}{2} \| \mb \delta_i \|_2^2 ) \mb \delta_i \\
&=& \mb b - \frac{t_{\mb b}}{c} ( \mb b - \mb b_\star ) \innerprod{ \One }{ \nu \dot{\mb
\psi} } - \frac{\nu t_{\mb b}}{c} ( \mb A - \mb A_\star ) \mb U \dot{\mb \psi},
\end{eqnarray*}
where above, we have set 
\begin{equation}
  \dot{\mb \Psi} = \left[ \begin{array}{ccc} \dot{\psi}(\tfrac{1}{2} \| \mb \delta_1 \|_2^2 ) &
    & \\ & \ddots & \\ & & \dot{\psi}(\tfrac{1}{2} \| \mb \delta_c \|_2^2) \end{array} \right] \in \bb R^{c \times c}, \qquad \dot{\mb \psi} = \left[ \begin{array}{c} \dot{\psi}(\tfrac{1}{2} \| \mb \delta_1 \|_2^2 ) \\ \vdots \\ \dot{\psi}(\tfrac{1}{2} \| \mb \delta_c \|_2^2) \end{array} \right] \in \bb R^{c}.
  \end{equation}
Writing $\mb \Delta_{\mb A} = \mb A - \mb A_\star$, $\mb \Delta_{\mb b} = \mb b - \mb b_\star$, we
have
\begin{eqnarray*}
\mb \Delta_{\mb A}^+ &=&  \mb \Delta_{\mb A} \Bigl( \mb I - \tfrac{\nu t_{\mb A}}{c} \mb U \dot{\mb \Psi} \mb U^*
\Bigr) - \tfrac{\nu t_{\mb A}}{c} \mb \Delta_{\mb b} \dot{\mb \psi}^* \mb U^* \\
\mb \Delta_{\mb b}^+ &=& \Bigl( 1 - \tfrac{t_{\mb b}}{c} \innerprod{ \One }{ \nu \dot{\mb \psi} }
\Bigr) \mb \Delta_{\mb b} - \mb \Delta_{\mb A} \tfrac{\nu t_{\mb b}}{c} \mb U \dot{\mb \psi}. 
\end{eqnarray*}
To facilitate a convergence proof, we modify this equation to pertain to scaled versions of $\mb \Delta_{\mb A}$, $\mb \Delta_{\mb b}$: 
\begin{eqnarray*}
t_{\mb A}^{-1/2} \mb \Delta_{\mb A}^+ &=& (t_{\mb A}^{-1/2} \mb \Delta_{\mb A}) \Bigl( \mb I - \tfrac{t_{\mb A}}{c} \mb U (\nu \dot{\mb \Psi}) \mb U^* \Bigr)  - \tfrac{t_{\mb A}^{1/2} t_{\mb b}^{1/2}}{c} ( t_{\mb b}^{-1/2} \mb \Delta_{\mb b} ) (\nu \dot{\mb \psi})^* \mb U^* \\
t_{\mb b}^{-1/2} \mb \Delta_{\mb b}^+ &=& \Bigl( 1 - \tfrac{t_{\mb b}}{c} \innerprod{ \One}{ \nu \dot{\mb \psi} } \Bigr) (t_{\mb b}^{-1/2} \mb \Delta_{\mb b}) - (t_{\mb A}^{-1/2} \mb \Delta_{\mb A}) \tfrac{t_{\mb b}^{1/2} t_{\mb A}^{1/2}}{c} \mb U (\nu \dot{\mb \psi}). 
\end{eqnarray*}
In matrix-vector form, and writing $\bar{\mb \Delta}_{\mb A} = t_{\mb A}^{-1/2} \mb \Delta_{\mb A}$ and $\bar{\mb \Delta}_{\mb b} = t_{\mb b}^{-1/2} \mb \Delta_{\mb b}$, we have
\begin{eqnarray}
    \left[ \begin{array}{c} \mr{vec}(\bar{\mb \Delta}_{\mb A}) \\ \bar{\mb \Delta}_{\mb b}
      \end{array} \right]^+ &=& \left( \mb I_6 - \left[ \begin{array}{cc} \tfrac{t_{\mb A}}{c} \mb U
    (\nu \dot{\mb \Psi}) \mb U^* \otimes \mb I_2 & \tfrac{t_{\mb A}^{1/2} t_{\mb b}^{1/2}}{c} (\mb U
(\nu \dot{\mb \psi}) ) \otimes \mb I_2 \\  \tfrac{t_{\mb A}^{1/2} t_{\mb b}^{1/2}}{c} ((\nu \dot{\mb
\psi})^* \mb U^* ) \otimes \mb I_2 & \tfrac{t_{\mb b}}{c} \innerprod{ \One}{ \nu \dot{\mb
\psi} } \otimes \mb I_2 \end{array} \right] \right) \left[ \begin{array}{c} \mr{vec}(\bar{\mb
\Delta}_{\mb A}) \\ \bar{\mb \Delta}_{\mb b} \end{array} \right] \nonumber \\ 
    &=& \left( \mb I_6 - \left[ \begin{array}{cc} \tfrac{t_{\mb A}}{c} \mb U (\nu \dot{\mb \Psi})
      \mb U^* & \tfrac{t_{\mb A}^{1/2} t_{\mb b}^{1/2}}{c} (\mb U (\nu \dot{\mb \psi}) )  \\
  \tfrac{t_{\mb A}^{1/2} t_{\mb b}^{1/2}}{c} ((\nu \dot{\mb \psi})^* \mb U^* ) & \tfrac{t_{\mb
  b}}{c} \innerprod{ \One}{ \nu \dot{\mb \psi} } \end{array} \right] \otimes \mb I_2
\right) \left[ \begin{array}{c} \mr{vec}(\bar{\mb \Delta}_{\mb A}) \\ \bar{\mb \Delta}_{\mb b}
\end{array} \right] \nonumber \\
    &\doteq& \mb M \left[ \begin{array}{c} \mr{vec}(\bar{\mb \Delta}_{\mb A}) \\ \bar{\mb \Delta}_{\mb b} \end{array} \right], \label{eqn:one-iter}
\end{eqnarray}
where in this context $\kron$ denotes the Kronecker product of matrices.
Since $\vI_6 = \vI_4 \kron \vI_2$, and because the eigenvalues of a Kronecker product of symmetric
matrices are the pairwise products of the eigenvalues of each factor,
we have 
\begin{align*}
  \| \mb M \|_{\ell^2 \to \ell^2} 
  &=
  \norm*{
    \vI - 
    \left[
      \begin{array}{cc} \tfrac{t_{\mb A}}{c} \mb U (\nu \dot{\mb \Psi})
        \mb U^* & \tfrac{t_{\mb A}^{1/2} t_{\mb b}^{1/2}}{c} (\mb U (\nu \dot{\mb \psi}) )  \\
        \tfrac{t_{\mb A}^{1/2} t_{\mb b}^{1/2}}{c} ((\nu \dot{\mb \psi})^* \mb U^* ) & \tfrac{t_{\mb
    b}}{c} \innerprod{ \One}{ \nu \dot{\mb \psi} } \end{array} \right]
  }_{\ell^2 \to \ell^2}.
\end{align*}
By our choice of $t_{\vA}$ and $t_{\vb}$, and the assumption $\vU \One = \Zero$, we can write
\begin{equation*}
  \left[
    \begin{array}{cc} \tfrac{t_{\mb A}}{c} \mb U (\nu \dot{\mb \Psi})
      \mb U^* & \tfrac{t_{\mb A}^{1/2} t_{\mb b}^{1/2}}{c} (\mb U (\nu \dot{\mb \psi}) )  \\
      \tfrac{t_{\mb A}^{1/2} t_{\mb b}^{1/2}}{c} ((\nu \dot{\mb \psi})^* \mb U^* ) & \tfrac{t_{\mb
  b}}{c} \innerprod{ \One}{ \nu \dot{\mb \psi} } \end{array} \right]
  =
  \left[
    \begin{array}{cc} \tfrac{t_{\mb A}}{c} \mb U (\nu \dot{\mb \Psi} - \vI)
      \mb U^* + \frac{\vU\vU\adj}{\norm{\vU}^2_{\ell^2\to\ell^2}} & \tfrac{t_{\mb A}^{1/2} t_{\mb
      b}^{1/2}}{c} (\mb U (\nu \dot{\mb \psi} - \One) )  \\
      \tfrac{t_{\mb A}^{1/2} t_{\mb b}^{1/2}}{c} ((\nu \dot{\mb \psi} - \One)^* \mb U^* ) & \tfrac{t_{\mb
  b}}{c} \innerprod{ \One}{ \nu \dot{\mb \psi} - \One } + 1 \end{array} \right]
\end{equation*}
and so by the triangle inequality for the operator norm
\begin{eqnarray}
\| \mb M \|_{\ell^2 \to \ell^2} &\le& \left\| \mb I - \left[ \begin{array}{cc} \frac{\mb U \mb
  U^*}{\| \mb U \|_{\ell^2\to\ell^2}^2} & \\ & 1  \end{array} \right] \right\|_{\ell^2 \to \ell^2} + \left\|\left[
  \begin{array}{cc} \frac{\mb U}{\norm{\vU}_{\ell^2\to\ell^2}} ( \nu \dot{\mb \Psi} - \mb I )
    \frac{\mb U\adj}{\norm{\vU}_{\ell^2\to\ell^2}} & \frac{1}{\sqrt{c}} \frac{\mb
      U}{\norm{\vU}_{\ell^2\to\ell^2}} (\nu \dot{\mb \psi}-\One) \\ \frac{1}{\sqrt{c}} (\nu
      \dot{\mb \psi}-\One)^* \frac{\mb U^*}{\norm{\vU}_{\ell^2\to\ell^2}} & 
      \innerprod{ \frac{\One}{c}}{ \nu
\dot{\mb \psi} - \One} \end{array} \right] \right\|_{\ell^2 \to \ell^2} \nonumber \\
&\le& 1 - \frac{1}{\kappa} + 2 \| \nu \dot{\mb \psi} - \One\|_{\ell^\infty}, \label{eqn:M-bound}
\end{eqnarray}
since
\begin{align*}
  \left\|\left[
    \begin{array}{cc} \frac{\mb U}{\norm{\vU}_{\ell^2\to\ell^2}} ( \nu \dot{\mb \Psi} - \mb I )
      \frac{\mb U\adj}{\norm{\vU}_{\ell^2\to\ell^2}} & \frac{1}{\sqrt{c}} \frac{\mb
      U}{\norm{\vU}_{\ell^2\to\ell^2}} (\nu \dot{\mb \psi}-\One) \\ \frac{1}{\sqrt{c}} (\nu
      \dot{\mb \psi}-\One)^* \frac{\mb U^*}{\norm{\vU}_{\ell^2\to\ell^2}} & 
      \innerprod{ \frac{\One}{c}}{ \nu
  \dot{\mb \psi} - \One} \end{array} \right] \right\|_{\ell^2 \to \ell^2}
  &\leq
  \left\|\left[
    \begin{array}{cc} \frac{\mb U \diag(\nu \dot{\vpsi} - \One)
      \vU\adj}{\norm{\vU}_{\ell^2\to\ell^2}^2} & \Zero \\ \Zero & 
      \innerprod{ \frac{\One}{c}}{ \nu
  \dot{\mb \psi} - \One} \end{array} \right] \right\|_{\ell^2 \to \ell^2} \\
  &\quad+
  \left\|\left[
    \begin{array}{cc} \Zero & \frac{1}{\sqrt{c}} \frac{\mb
      U (\nu \dot{\mb \psi}-\One)}{\norm{\vU}_{\ell^2\to\ell^2}} \\ 
\frac{1}{\sqrt{c}}  \frac{(\nu \dot{\mb \psi}-\One)^* \mb U^*}{\norm{\vU}_{\ell^2\to\ell^2}}
      & 
      \Zero \end{array} \right] \right\|_{\ell^2 \to \ell^2}
\end{align*}
and by H\"{o}lder's inequality
\begin{equation*}
  \abs{\ip{ \tfrac{\One}{c}}{ \nu \dot{\mb \psi} - \One} } \leq \norm{\nu \dot{\mb \psi} -
  \One}_{\ell^\infty},
  \quad
  \tfrac{1}{\sqrt{c}} \norm{\nu \dot{\mb \psi}-\One}_{\ell^2} \leq \norm{\nu \dot{\mb \psi} -
  \One}_{\ell^\infty}.
\end{equation*}

\paragraph{Inductive argument for \eqref{eqn:opt-main-bound}.} We begin by noting that since $\mb
A_0 = \mb I$, $\mb b_0 = \mb 0$, \eqref{eqn:opt-main-bound} holds for $k =0$. Now assume that it
is true for $0,1,\dots, k-1$. If we can verify that 
\begin{equation} \label{eqn:psi-cond}
    2 \| \nu \dot{\mb \psi} - \One\|_{\ell^\infty} \le \frac{1}{2\kappa},
\end{equation}
then by \eqref{eqn:one-iter} and \eqref{eqn:M-bound} together with $t_{\vb}=1$, we have 
\begin{eqnarray*}
    \left\| \left[ \begin{array}{c} t_{\mb A}^{-1/2} \mr{vec}( \mb A_k - \mb A_\star ) \\ \mb b_k -
      \mb b_\star \end{array} \right]\right\|_F^2 &\le& \left( 1 - \frac{1}{2 \kappa} \right)^2
      \left\| \left[ \begin{array}{c} t_{\mb A}^{-1/2} \mr{vec}( \mb A_{k-1} - \mb A_\star ) \\ \mb b_{k-1} - \mb b_\star \end{array} \right]\right\|_F^2.
\end{eqnarray*}
Applying the inductive hypothesis, we obtain \eqref{eqn:opt-main-bound} for iteration $k$. So,
once we can show that under the inductive hypothesis, \eqref{eqn:psi-cond} holds, the result will
be established. 

We begin by showing that under the inductive hypothesis, the errors $\mb \delta_i$ are all
bounded. Indeed, by the parallelogram law
\begin{eqnarray*}
 \| \mb \delta_i \|_2^2 &=& \| \mb A_{k-1} \mb u_i + \mb b_{k-1} - \mb v_i \|_2^2 \\
 &=& \left\| (\mb A_{k-1} - \mb A_\star ) \mb u_i + (\mb b_{k-1} - \mb b_\star) \right\|_2^2 \\
 &\le& 2 \frac{\| \mb A_{k-1} - \mb A_\star\|_{F}^2}{t_{\mb A}} t_{\mb A} \| \mb u_i \|_2^2 + 2 \| \mb b_{k-1} - \mb b_\star \|_2^2, 
\end{eqnarray*}
and so applying the inductive hypothesis to bound 
\begin{equation*}
    t_{\mb A}^{-1} \| \mb A_{k-1} - \mb A_\star \|_F^2 + \| \mb b_{k-1} - \mb b_\star \|_2^2 \le t_{\mb A}^{-1} \| \mb I - \mb A_\star \|_F^2 + \| \mb b_\star \|_2^2,
\end{equation*}
we obtain for all $i$
\begin{eqnarray}
    \| \mb \delta_i \|_2 &\le& \sqrt{2} \times \sqrt{ t_{\mb A}^{-1} \|\mb A_\star - \mb I \|_F^2
    + \| \mb b_\star \|_2^2} \times \max\left\{ t_{\mb A}^{1/2} \| \mb u_i \|_2, 1 \right\}
    \nonumber \\ 
     &\le& \sqrt{2} \times \sqrt{ t_{\mb A}^{-1} \|\mb A_\star - \mb I \|_F^2 + \| \mb b_\star
     \|_2^2} \times \frac{\sqrt{c} \| \mb U \|_{\ell^1\to \ell^2}}{ \| \mb U \|_{\ell^2 \to
   \ell^2} }. \label{eqn:delta-bound}
\end{eqnarray}
Since 
\begin{equation*}
  \psi(s) = \frac{1}{4\pi\sigma^2}\left(1 - \exp\left( - \frac{s}{2 \sigma^2} \right)\right),
\end{equation*}
we have 
\begin{equation*}
    \dot{\psi}(s) = \frac{1}{8 \pi \sigma^4} \exp\left( - \frac{s}{2 \sigma^2} \right),
\end{equation*}
and for all $s \geq 0$
\begin{equation*}
    \left| 1 - \nu \dot{\psi}(s) \right| = \left| 1 - 8 \pi \sigma^4 \dot{\psi}(s) \right| \le
    \frac{s}{2 \sigma^2}
\end{equation*}
by the standard exponential convexity estimate.
Plugging in our bound \eqref{eqn:delta-bound}, we obtain for all $i$
\begin{eqnarray*} 
\left| 1 - \nu \dot{\psi}( \tfrac{1}{2} \| \mb \delta_i \|_2^2 ) \right| &\le& \frac{ t_{\mb
A}^{-1} \| \mb A_\star - \mb I \|_F^2 + \| \mb b_\star \|_2^2 }{2\sigma^2} \times \frac{c \| \mb U
\|_{\ell^1\to \ell^2}^2 }{ \| \mb U \|_{\ell^2 \to \ell^2}^2 }. 
\end{eqnarray*}
Under our choice of $t_{\vA}$ and hypotheses on $\sigma$, this is bounded by $\tfrac{1}{4\kappa}$. 
\end{proof}

\end{document}